\documentclass[11pt]{article}

\usepackage{xr}
\usepackage[letterpaper,margin=1.5in]{geometry}
\usepackage{amsthm}
\usepackage{amsmath}
\usepackage{amsfonts,amssymb}
\usepackage[titletoc]{appendix}
\usepackage{float}
\usepackage[shortlabels]{enumitem}
\usepackage{graphicx}
\usepackage{setspace}
\usepackage{tikz}
\usepackage{tkz-euclide}
\usepackage{algorithm,algorithmic}
\usepackage{bm}
\usepackage{subfig}
\usetikzlibrary{calc}
\usetikzlibrary{fadings}
\usetikzlibrary{arrows}

\newcommand{\assign}{\leftarrow}

\newcommand{\citep}[1]{\cite{#1}}
\newcommand{\citet}[1]{\cite{#1}}
\newcommand{\yrcite}[1]{\cite{#1}}
\newcommand{\mode}{\operatorname{mode}}

\newcommand{\spc}{\hspace{1cm}}
\newcommand{\Bin}{\operatorname{Bin}}
\newcommand{\Ber}{\operatorname{Ber}}
\newcommand{\OC}[1]{$\operatorname{OC}({#1})$}
\newcommand{\DC}[1]{$\operatorname{DC}_{#1}$}
\newcommand{\supp}{\operatorname{supp}}

\newtheorem{theorem}{Theorem}
\newtheorem{corollary}[theorem]{Corollary}
\newtheorem{lemma}[theorem]{Lemma}
\newtheorem{definition}{Definition}
\newtheorem{condition}{Condition}
\newcommand{\SUBALGNAME}{{SCIP}}
\newcommand{\subalgname}{\operatorname{SCIP}}
\newcommand{\prob}{\mathbb{P}}
\newcommand{\BG}{\operatorname{\mathsf{BG}}}

\newcommand{\cE}{\mathcal{E}}

\newcommand{\cI}{\mathcal{I}}

\newcommand{\cN}{\mathcal{N}}

\newcommand{\cR}{\mathcal{R}}
\newcommand{\cS}{\mathcal{S}}

\newcommand{\cW}{\mathcal{W}}

\newcommand{\essenS}{\underline{I}}
\newcommand{\sgn}{\operatorname{sign}}
\newcommand{\ba}{\mathbf{a}}

\newcommand{\bA}{\mathbf{A}}

\newcommand{\bD}{\mathbf{D}}

\newcommand{\bI}{\mathbf{I}}

\newcommand{\bM}{\mathbf{M}}
\newcommand{\bN}{\mathbf{N}}

\newcommand{\bS}{\mathbf{S}}

\newcommand{\bX}{\mathbf{X}}

\newcommand{\bbE}{\mathbb{E}}

\newcommand{\bbZ}{\mathbb{Z}}

\newcommand{\bSigma}{\boldsymbol{\Sigma}}

\newcommand{\ol}[1]{\overline{#1}}
\newcommand{\R}{{\rm I}\kern-0.18em{\rm R}}
\newcommand{\h}{{\rm I}\kern-0.18em{\rm H}}
\newcommand{\PP}{{\rm I}\kern-0.18em{\rm P}}
\newcommand{\E}{{\rm I}\kern-0.18em{\rm E}}
\newcommand{\Z}{\mathbb Z}
\newcommand{\1}{{\rm 1}\kern-0.24em{\rm I}}
\newcommand{\N}{{\rm I}\kern-0.18em{\rm N}}

\DeclareMathOperator*{\argmax}{argmax}

\usepackage[utf8]{inputenc} 
\usepackage{hyperref}       
\usepackage{url}            
\usepackage{booktabs}       
\usepackage{amsfonts}       

\begin{document}
\title{\bf Sparse Gaussian ICA}
\author{
  Nilin Abrahamsen\qquad
  Philippe Rigollet
  \vspace{0.15in}\\Department of Mathematics, Massachusetts Institute of Technology}
\maketitle

\begin{abstract}Independent component analysis (ICA) is a cornerstone of modern data analysis. Its goal is to recover a latent random vector $S$ with independent components from samples of $X=\bA S$ where $\bA$ is an unknown mixing matrix. Critically, all existing methods for ICA rely on and exploit strongly the assumption that $S$ is not Gaussian as otherwise $\bA$ becomes unidentifiable. In this paper, we show that in fact one can handle the case of Gaussian components by imposing structure on the matrix $\bA$. Specifically, we assume that $\bA$ is sparse and generic in the sense that it is generated from a sparse Bernoulli-Gaussian ensemble. Under this condition, we give an efficient algorithm to recover the columns of $\bA$ given only the covariance matrix of $X$ as input even when $S$ has several Gaussian components.\end{abstract}
  
\section{Introduction}

Independent component analysis (ICA) is a statistical model which has become ubiquitous in a variety of applications including image processing~\cite{BelSej97,ZhaSunLiu07}, neuroscience~\cite{JunMakMck01}, and genomics~\cite{Sur03,KonVanGun08,EngDaiMar10}. The ICA model expresses an observed random vector $X\in\R^r$ as a linear transformation
\[X=\bA S\label{sparseICA}\]
of a latent random vector $S\in\R^s$ with independent components, called \emph{sources}. Here, $\bA\in\R^{r\times s}$ is an unknown deterministic \emph{mixing matrix}~\cite{HyvKarOja01}. Arguably the most studied problem in the ICA model is \emph{blind source separation} where the goal is to recover both the mixing matrix and the sources from observations of $X$. Another problem, called \emph{feature extraction}, is that of recovering just the mixing matrix~\cite{HyvOja98}. In this paper we focus on the feature extraction problem. Unlike blind source separation, feature extraction may be solved even in the \emph{overcomplete} setting where $s>r$. Nevertheless, for $\bA$ to be identifiable, additional assumptions need to be imposed beyond independence of the latent components. Indeed, if $S\sim\cN(0,\bI_s)$, then $\bSigma=\bbE XX^\top=\bA\bA^\top$ is a sufficient statistic for the distribution of $X$, and this is unchanged if $\bA$ is replaced by $\bA U$ for any orthogonal matrix $U\in\R^{s\times s}$. Thus, $\bA$ is at best identifiable up to right multiplication by an orthogonal matrix. It turns out that the above example is essentially the only case when $\bA$ is not identifiable. More precisely, a classical result states that if at most one component of $S$ is Gaussian, then $\bA$ can be recovered up to a permutation and rescaling of its columns~\cite{Com94}.

\paragraph{Previous work.}
In view of the identifiability issues arising in the Gaussian case, practical algorithms for ICA have traditionally relied on the fourth cumulants of $X$ to exploit non-Gaussianity~\cite{AlbFerChe05}. 
Perhaps the most widely known method in this line of work is the FastICA algorithm by Hyv\"arinen and Oja~\yrcite{HyvOja97}, which iteratively finds the one-dimensional projections maximizing kurtosis of the data. However, all such methods fail when $S$ has fourth moments close to those of a Gaussian. Moreover, traditional methods for ICA almost universally use of an initial \emph{whitening} step, which transforms $X$ to have covariance matrix $\bI_r$. This step is fragile to independent additive noise on $X$. Voss et al.~\yrcite{BelRadVos13} introduce algorithms to overcome this problem, and Arora et al.~\yrcite{AroGeMoi15} use a \emph{quasi-whitening} step which allows them to prove guarantees for ICA in the presence of additive Gaussian noise. Nevertheless, all these methods exploit non-gaussianity of the sources.

\paragraph{Our contribution.}
In this paper we take a radically different approach that removes distributional assumptions on $S$. In fact we do not even require $S$ to have independent components, but only that the components be uncorrelated. In addition our methods are robust to additive independent noise on $X$ with any centered distribution. We achieve this by instead making structural assumptions on~$\bA$. Specifically, we assume that $\bA$ is \emph{sparse} and \emph{generic}. Sparsity has been a key idea to reduce dimensionality in signal processing and statistics~\cite{CanTao05,FouRau13,HasTibWai15,BuhVDGee11}.
Beyond reducing complexity and avoiding overfitting, sparsity of the mixing matrix is of great practical interest because it leads to \emph{interpretable} results. Hyv\"arinen and Raju \yrcite{HyvRaj02} have previously proposed to impose sparsity conditions on the mixing matrix in order to improve the performance of ICA. This work presents the first rigorous treatment of an algorithm taking advantage of a sparse mixing matrix. 

In addition to being sparse we require that $\bA$ be generic, which we enforce by choosing $\bA$ as the realization of a Bernoulli-Gaussian ensemble. Similar structural assumptions have recently been employed in dictionary learning for example~\cite{SpiWanWri12, AroGeMoi13, AgaAnaJai14}. While the two problems are related, fundamental differences preclude the use of standard dictionary learning machinery (See section~\ref{model}).
\paragraph{Notation.}
We use the shorthand $[r]=\{1,\ldots,r\}$. We write the $i$th entry of vector $v$ as $v(i)$. For $v\in\R^r$ and indices $I=\{i_1,\ldots,i_k\}\subset[r]$, we write the restriction of $v$ to $I$ as $v(I)=(v(i_1),\ldots,v(i_k))^\top$. $e_i\in\R^r$ is the $i$th standard unit vector, $e_i(i)=1$ and $e_i(j)=0$ for $j\neq i$. Similarly for a matrix $M$, $M(I\times J)$ is the submatrix $(M_{ij})_{i\in I,j\in J}$. $M(\cdot\times J)$ is the submatrix which keeps all rows but only columns indexed by $j\in J$. We say that a matrix is \emph{fully dense} if all its entries are nonzero. For tuples $v$ and $w$, $v/w$ is the entrywise ratio, $(v/w)(i)=v(i)/w(i)$. $diag(v)$ is the diagonal matrix with $v$ along the diagonal.
For a set $\cS$, $|\cS|$ is its cardinality. We define the \emph{support} of a vector $\supp v=\{i|v(i)\neq0\}$ and write $|v|_0=|\supp v|$. $|v|_p=(\sum_i|v(i)|^p)^{1/p}$ denotes the $p$-norm, and $|M|_\infty=\sup_{ij}|M_{ij}|$ is the entrywise supremum norm.
$\cW(\bSigma,n)$ denotes the Wishart distribution with scale matrix $\bSigma\in\R^{r\times r}$ and $n$ degrees of freedom, i.e., for i.i.d. samples $X_1,\ldots,X_n\sim \cN(0,\bSigma)$ of a Gaussian vector, $\frac1n\sum_{i=1}^nX_iX_i^\top\sim\cW(\bSigma,n)$. $\Bin(n,\theta)$ is the binomial distribution with $n$ trials and success probability $\theta$, and $\Ber(\theta)=\Bin(1,\theta)$ is the distribution of a Bernoulli trial with expectation $\theta$. We write the identity matrix in $\R^{r\times r}$ as $\bI_r$ and the all-ones vector as $\1=\1_r$. We use the notation $x\vee y=\max\{x,y\}$ and $x\wedge y=\min\{x,y\}$.

\section{Statistical Model}\label{model}
Let $S\in\R^s$ be a random vector of sources with independent components. $S$ is transformed by multiplication with the unknown mixing matrix $\bA\in\R^{r\times s}$ to be estimated. We also allow independent additive noise $N\sim\cN(0,\bD_\sigma)$ on the transformed vector, where $\bD_\sigma=diag(\sigma_1^2,\ldots,\sigma_r^2)$ is some diagonal nonnegative matrix (possibly zero). We arrange $n$ i.i.d. copies of $S$ in a matrix $\bS\in\R^{s\times n}$ whose entries are i.i.d. with distribution $\cN(0,1)$.\footnote{The Gaussianity assumption can be relaxed.} Similarly, concatenate $n$ copies of $N$ to construct $\bN\in\R^{r\times n}$ independent of $\bS$ and with independent entries $\bN_{im}\sim\cN(0,\sigma_i^2)$. The observed data is then $\bX\in\R^{r\times n}$ given by
\begin{equation}
\bX=\bA\bS+\bN.\label{finsample}
\end{equation}
The columns of $\bX$ are i.i.d.~with distribution $\cN(0,\bSigma)$, where
\[\bSigma=\bA\bA^\top+\bD_\sigma.\]
Write the sample covariance matrix as $\overline\bSigma=\frac1n\bX\bX^\top$. Then $\ol\bSigma\sim \cW(\bSigma,n)$ follows a Wishart distribution with scale matrix $\bSigma$ and $n$ degrees of freedom. Our goal is to learn $\bA$ up to permutations and sign changes of its columns, i.e., to recover $\bA\Pi \Delta$ where $\Pi$ is some permutation matrix, and $\Delta$ is a diagonal matrix with diagonal entries in $\{-1,1\}$.

\paragraph{Relation to dictionary learning.}
Dictionary learning, also known as sparse coding \citep{MaiBacPon10}, is a matrix factorization problem which is formally equivalent with~\eqref{finsample} without the noise term. With our notation the problem can be stated as follows. An unknown matrix $\bS^\top\in \R^{n\times s}$, called the \emph{dictionary}, is assumed to have various properties for identifiability purposes. These include \emph{incoherence}~\citep{AroGeMoi13,AgaAnaJai14}, or invertibility~\citep{SpiWanWri12}. The columns of $\bS^\top$ are called the \emph{atoms}. A sequence of $r$ vectors is observed, each of which is a sparse linear combination of atoms. Appending the observed vectors yields a matrix
\[\bX^\top=\bS^\top \bA^\top\,,\]
where $\bA^\top\in\R^{s\times r}$ is sparse and generic. The task is to recover $\bA^\top$ and the dictionary $\bS^\top$.
While this problem is formally equivalent with \eqref{finsample}, dictionary learning traditionally treats the regime $r>s$, i.e., the number of samples is larger than the number of dictionary elements~\citep{SpiWanWri12, LuhVu16}. This assumption is overly restrictive for our purposes as we allow the number of features $s$ to exceed the number of observed variables $r$, so we cannot employ existing results on dictionary learning. More specifically, while our results also cover the case $r>s$, we are primarily interested in the regime where $r\leq s$. 

In order to ensure that the mixing matrix is generic we generate $\bA$ by the following random model.
\begin{definition}\label{BG}
Let $\big(B_{ij},\xi_{ij}\big)_{\substack{i=1,\ldots,r\\j=1,\ldots,s}}$ be mutually independent random variables where
\[B_{ij}\sim \Ber(\theta)\,\text{, and }\xi_{ij}\sim \cN(0,1)\,\text{.}\]
If $A_{ij}=B_{ij}\xi_{ij}$, we say that matrix $\bA=(A_{ij})_{ij}$ arises from a Bernoulli-Gaussian ensemble and write $\bA\sim \BG(r,s,\theta)$.
\end{definition}
The mixing matrix $\bA\sim \BG(r,s,\theta)$ has, in expectation $rs\theta$ entries and we refer to $\theta$ has the \emph{sparsity parameter}. It represents the fraction of nonzero-entries. Let $\ba_1,\ldots, \ba_s$ and $\rho_1,\ldots, \rho_r$ denote the columns and rows of $\bA$ respectively,
\[\bA=(\ba_1\ldots \ba_s)=(\rho_1\ldots\rho_r)^\top\text{.}\]

Because of ambiguities from permutations and sign changes of $\ba_1,\ldots,\ba_s$ we evaluate the performance of the recovery algorithm in terms of the following distance measure.
\begin{definition}
Given $\bA,\hat\bA\in \R^{r\times s}$, write
\[d(\hat\bA,\bA)=\min_{\Pi,\Delta}|\hat\bA\Pi \Delta-\bA|_\infty,\]
where $\Pi\in \R^{s\times s}$ ranges over all permutation matrices, and $\Delta\in\R^{s\times s}$ ranges over diagonal matrices with diagonal entries taking values in $\{-1,1\}$.
\end{definition}

\newpage
\section{Main result}
We prove that when $\bA\sim\BG(r,s,\theta)$ and the sparsity parameter is of order $\theta\ll s^{-1/2}$, $\bA$ can be efficiently recovered from the covariance matrix $\bSigma$. In this setting, at most a small constant fraction of the entries of $\bSigma$ are nonzero. Moreover, we show that when $n$ is of order $s^2$, the sample covariance matrix $\overline\bSigma$ suffices to approximately recover $\bA$.
This is our main theorem.
\newcommand\mnthm{
  There exist $c,C>0$ such that the following holds. Let $r,s,\theta$ be such that
    \[C\frac{\log(r/\delta)}{r}\leq\theta\leq\frac{c}{\sqrt{s}+\log(r/\delta)},\]
    Let $\bSigma=\bA\bA^\top+\bD_\sigma$ where $\bA\sim\BG(r,s,\theta)$ and $\bD_\sigma$ is diagonal. Put
    \[n=C|\bSigma|_\infty^2(s^4\theta^6+(\log\tfrac r\delta)^2)\log(r/\delta).\]
    Then there is a randomized algorithm outputting $\hat\bA$ on input $\overline\bSigma\sim \cW(\bSigma,n)$ in expected time $O(r\theta rs)$ such that with probability $1-\delta$ over the randomness of $\bA$ and $\ol\bSigma$, $d(\hat\bA,\bA)=O(|\bA|_\infty^2|\bSigma|_\infty\sqrt{\log(r/\delta)/n})$.
  }
\begin{theorem}\label{thm:sampletheorem}\mnthm\end{theorem}
The quantity $|\bSigma|_\infty$ can also be characterized in terms of the largest squared norm of a row of $\bA$, i.e., $|\bSigma|_\infty=|\bA\bA^\top|_\infty+|\bD_\sigma|_\infty =\max_{i}|\rho_i|_2^2+\sigma_i^2$. 
Coupling $r,s$, and $\theta$ yields the following theorem, which gives a qualitative illustration of theorem \ref{thm:sampletheorem}. This asymptotic result uses a bound stating that $\max_{i}|\rho_i|_2^2$ concentrates around $\E|\rho_i|_2^2=s\theta$. 
\begin{corollary}\label{cor:simplifiedmain}
  Let $\bA\sim\BG(s,s,\theta)$, and let $\bSigma=\bA\bA^\top$, where
  \[\theta=s^{-\alpha}\]
for some fixed exponent $\frac12<\alpha<1$, then there is an algorithm which takes input $\bSigma$ and outputs $\hat\bA$, such that $d(\hat\bA-\bA)\to 0$ in probability over the randomness of $\bA$ as $s\to\infty$.
Moreover, the same guarantee holds for input $\ol\bSigma\sim\cW(\bSigma,n)$ if $n={s^{2\beta}}$, where
    \[\beta>(3-4\alpha)\vee (1-\alpha).\]
\end{corollary}
Theorem \ref{thm:sampletheorem} and corollary \ref{cor:simplifiedmain} are proven in section \ref{guarantee_section}.

\begin{figure}[H]
\centering
\resizebox{0.7\columnwidth}{!}{

\definecolor{maroon}{RGB}{163,31,52}
\definecolor{grey}{RGB}{240,240,240}

\begin{tikzpicture}

\begin{scope}[scale=4.5]

\fill [color=grey](0,0)rectangle(1,1);
\fill [color=maroon](0.5,1)--(2/3,1/3)--(1,0)--(1,1)--cycle;

\draw [->](0,0)--(1.1,0);
\draw [->](0,0)--(0,1.1);

\draw [dashed,color=gray](0,1)--(1,0);
\draw [dotted](0.5,0)--(0.5,1);
\draw [dotted](2/3,0)--(2/3,1/3);
\draw [dotted](0,1/3)--(2/3,1/3);

\node [right] at (1.1,0){$\alpha=\log_s\frac{1}{\theta}$};
\node [above] at (0,1.1){$\beta=\frac12\log_s n$};

\draw (0,-0.01)--(0,0.01);
\draw (0.5,-0.01)--(0.5,0.01);
\draw (2/3,-0.01)--(2/3,0.01);
\draw (1,-0.01)--(1,0.01);
\node [below] at (0,0){$0$};
\node [below] at (0.5,0){$\frac12$};
\node [below] at (2/3,0){$\frac23$};
\node [below] at (1,0){$1$};

\draw (-0.01,0)--(0.01,0);
\draw (-0.01,1/3)--(0.01,1/3);
\draw (-0.01,1)--(0.01,1);
\node [left] at (0,0){$0$};
\node [left] at (0,1/3){$\frac13$};
\node [left] at (0,1){$1$};

\node[below,rotate=-45] at (0.3,0.7){{\tiny $\beta=1-\alpha$}};

\end{scope}

\end{tikzpicture}

}
\caption{Sample complexity $n$ of recovering $\bA\sim\BG(cs,s,\theta)$ where ${\theta=s^{-\alpha}}$ and $n=s^{2\beta}$. The expected support size of a column of $\bA$ is $\bbE|\ba|_0=r^{1-\alpha}$, so the dashed line $\beta=1-\alpha$ represents choosing $n$ of the order $(\bbE|\ba|_0)^2$.}
\end{figure}
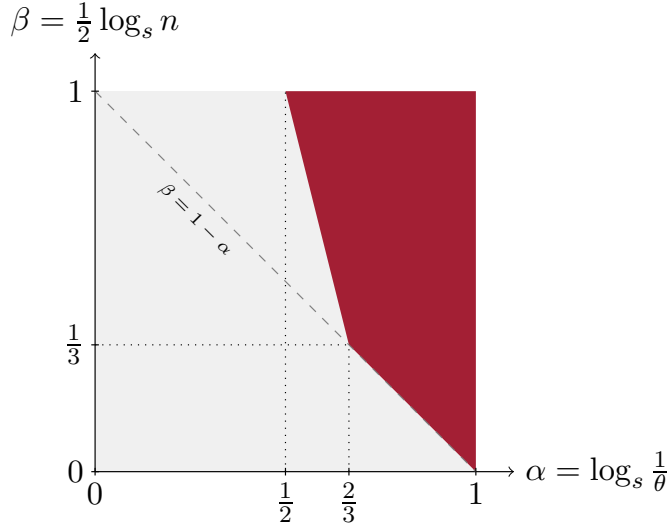


\section{Algorithm}
\label{errorless}
We now describe our algorithm, beginning in the population setting where $\bSigma$ is known exactly. The input to our algorithm is $\bSigma$, whose $i$th row we write as
\[\gamma_i^\top=e_i^\top\bSigma\]
We recover the columns of $\bA$ one at a time by applying the following Single Column Identification Procedure ($\subalgname$). $\subalgname$ takes as input $\bSigma$ and a pair of indices $(i_1,i_2)$, where $\Sigma_{i_1i_2}$ is a randomly chosen nonzero off-diagonal entry of $\bSigma$, and it outputs a column $\ba_j$, possibly with its sign changed. Denote the (unknown) supports of the columns and rows of $\bA$ by
\[I_j=\supp \ba_j\,\text{ and } R_i=\supp\rho_i\,\text{.}\]
$\Sigma_{i_1i_2}\neq0$, implies that $R_{i_1}\cap R_{i_2}\neq\emptyset$. Equivalently, $i_1,i_2\in I_j$ for some $j$. Using this fact, and supposing that there is only one such $j$ (this turns out to be the typical situation), $\subalgname$ outputs $\hat\ba=\pm\ba_j$. $\subalgname$ proceeds in two steps on input $\bSigma,(i_1,i_2)$.

The first step of $\subalgname$ finds a subset $L\subset I_j$ containing a large fraction of the unknown support $I_j$. To illustrate how this step works, assume $\bD_\sigma$ is the zero matrix\footnote{$\bD$ does not affect the output of $\subalgname$ because of its sparsity, so we assume $\bD_\sigma=0$ for notational convenience.}, and write $\bSigma$ as a sum of (unknown) rank one matrices
\begin{equation}\label{eq:rankone}\bSigma=\bM_1+\ldots+\bM_s\end{equation}
where $\bM_k=\ba_k\ba_k^\top$ for $k\in[s]$. Note that $\supp\bM_j=I_j\times I_j$. Then the $2$-by-$|I_j|$ matrix $\bM_j(\{i_1,i_2\}\times I_j)$ is a fully dense matrix of rank one. It turns out that because the supports of the matrices $\bM_1,\ldots,\bM_s$ have small overlaps, this property is approximately preserved when adding the contributions from the other $s-1$ terms in \eqref{eq:rankone}. That is, $\bSigma(\{i_1,i_2\}\times I_j)$ agrees $\bM_j(\{i_1,i_2\}\times I_j)$ on all but a small number of entries. Hence, $\bSigma(\{i_1,i_2\}\times I_j)$ can be made to have rank one by removing the columns where these entries appear. Equivalently, letting $L$ be the indices of the remaining columns, we get that $L\subset I_j$ is a set of indices such that $\bSigma(\{i_1,i_2\}\times L)$ is a rank one fully dense matrix. Therefore we can identify a large subset $L\subset I_j$ by picking the largest set $L$ such that $\bSigma(\{i_1,i_2\}\times L)$ is fully dense and has rank one. Another way of formulating this is that we define $L$ to be the largest set of indices such that $\gamma_{i_1}(L)$ and $\gamma_{i_2}(L)$ are fully dense and collinear. This concludes step 1 of $\subalgname$.

We now begin step 2 of $\subalgname$. At this stage we have $\gamma_{i_1}(L)=\lambda\ba_j(L)$ where $\lambda=\ba_j(i_1)$, so we could already make a crude approximation to $\lambda\ba_j$ by extending $\gamma_{i_1}(L)$ with zeroes outside of $L$. However, this approximation misses the entries in $I_j\setminus L$. Instead, consider the identity $\bM_j(\cdot\times L)diag(\ba_j(L))^{-1}=\ba_j\1^\top$. $\gamma_{i_1}$ is known to us and $\gamma_{i_1}(L)=\lambda\ba_j(L)$, so we can rewrite the identity as
\begin{equation}\label{eq:copies}\bM_j(\cdot\times L)diag(\gamma_{i_1}(L))^{-1}=\frac1\lambda\ba_j\1^\top,\end{equation}
The RHS is just $|L|$ copies of $\lambda^{-1}\ba_j$ written side by side. So if $\bM_j(\cdot\times L)$ were known, then we could easily recover $\ba_j$ up to a scalar by taking any column of \eqref{eq:copies}. It turns out that replacing $\bM_j$ by $\bSigma$ in \eqref{eq:copies} changes only a small fraction of the entries in each row. Hence, each row of
\[\lambda\bSigma(\cdot\times L)diag(\gamma_{i_1}(L))^{-1}-\ba_j\1^\top\]
is sparse. Moreover, both $\bSigma(\cdot\times L)$ and $diag(\gamma_{i_1}(L))^{-1}$ are known to us, so we can compute $\bSigma(\cdot\times L)diag(\gamma_{i_1}(L))^{-1}$. Now it is easy to compute $\tilde\ba=\lambda^{-1}\ba_j$, since its $i$th entry is repeated several times in the $i$th row of $\bSigma(\cdot\times L)diag(\gamma_{i_1}(L))^{-1}$. This row can also be written $\gamma_i(L)/\gamma_{i_1}(L)$. We identify $|\lambda|=|\ba_j(i_1)|$ using the fact that $\Sigma_{i_1i_2}=\ba_j(i_1)\ba_j(i_2)$ when $R_{i_1}\cap R_{i_2}=\{j\}$, and we output $|\lambda|\tilde\ba=\pm\ba_j$.
We have motivated the following procedure.
\paragraph{Step 1.}
Take input $\bSigma=(\gamma_1,\ldots,\gamma_r)^\top$ and $(i_1,i_2)$. Intersect the supports of the $i_1$st and $i_2$nd rows of $\bSigma$ and store the resulting set of indices as $K\subset[r]$. 
Compute the mode (most frequent value) $\hat\varphi$ of the entrywise ratio $\gamma_{i_2}(K)/\gamma_{i_1}(K)=(\gamma_{i_2}(k)/\gamma_{i_1}(k))_{k\in K}$. Then let $L\subset K$ be the set of indices where the mode is attained.
\paragraph{Step 2.}
Restrict attention to the submatrix $\Sigma(\cdot\times L)$ consisting of the columns indexed by $L$. Construct vector $\tilde\ba\in\R^r$ by defining $\tilde\ba(i)$ to be the median of $\gamma_i(L)/\gamma_{i_1}(L)$ for each $i=1,\ldots,r$. Multiply $\tilde\ba$ by a scalar and output the resulting rescaled vector $\hat\ba$. 

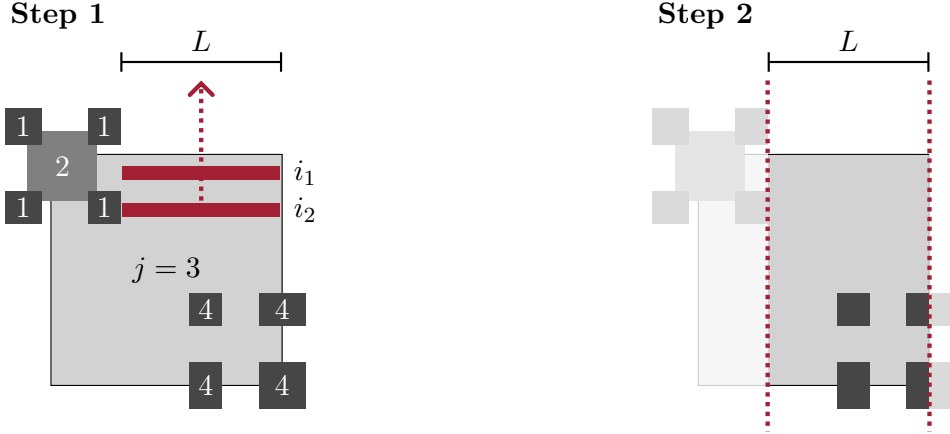
\begin{figure}
  \begin{center}
 \resizebox{\columnwidth}{!}{

\definecolor{maroon}{RGB}{163,31,52}
\definecolor{grey}{RGB}{210,210,210}
\definecolor{darkgray}{RGB}{70,70,70}
\newcommand{\drawmatrix}[4]{\draw [#3](#1,10-#1) rectangle (#2,10-#2)node[midway,color=black,opacity=1]{#4};
}
\newcommand{\fillmatrix}[4]{\fill [#3] (#1,10-#1) rectangle (#2,10-#2)node[midway,color=white,opacity=1]{#4};
}
\newcommand{\filldiscontinuous}[6]{
\fill [#5] (#1,10-#1) rectangle (#2,10-#2)node[midway,color=white,opacity=1]{#6};
\fill [#5] (#3,10-#3) rectangle (#4,10-#4)node[midway,color=white,opacity=1]{#6};

\fill [#5] (#1,10-#3) rectangle (#2,10-#4)node[midway,color=white,opacity=1]{#6};
\fill [#5] (#3,10-#1) rectangle (#4,10-#2)node[midway,color=white,opacity=1]{#6};
}
\newcommand*{\thatheight}{0.15}

\begin{tikzpicture}
  
\begin{scope}[scale=0.6]
\drawmatrix{2}{7}{fill=grey}{$j=3$};
\fillmatrix{1.5}{3}{color=gray}{$2$};

\filldiscontinuous{1}{1.8}{2.8}{3.5}{color=darkgray}{$1$};
\filldiscontinuous{5}{5.7}{6.5}{7.5}{color=darkgray}{$4$};

\draw [arrows={-angle 90},color=maroon,ultra thick,dotted] (5.25,10-3)--(5.25,9.6);

\fill [thick,color=maroon] (3.55,10-2.4+\thatheight) rectangle (6.95,10-2.4-\thatheight);
\fill [thick,color=maroon] (3.55,10-3.2+\thatheight) rectangle (6.95,10-3.2-\thatheight);


\draw [thick,|-|] (3.5,10)--(7,10) node [midway,above]{$L$};

\node at (7.5,10-2.4){$i_1$};
\node at (7.5,10-3.2){$i_2$};
\node at (0.9,11)[right]{\bf Step 1};
\end{scope}

\begin{scope}[scale=0.6,xshift=14cm]

\drawmatrix{2}{7}{fill=grey}{};
\fillmatrix{1.5}{3}{color=gray}{};

\filldiscontinuous{1}{1.8}{2.8}{3.5}{color=darkgray}{};
\filldiscontinuous{5}{5.7}{6.5}{7.5}{color=darkgray}{};


\fill[color=white,opacity=0.8](1,1.8)rectangle(3.5,10);
\fill[color=white,opacity=0.8](7,1.8)rectangle(7.5,10);

\draw [dotted,ultra thick,color=maroon] (3.5,9.6)--(3.5,2);
\draw [dotted,ultra thick,color=maroon] (7,9.6)--(7,2);

\draw [thick,|-|] (3.5,10)--(7,10) node [midway,above]{$L$};
\node at (0.9,11)[right]{\bf Step 2};

\end{scope}

\end{tikzpicture}

}
\end{center}
\caption{{\bf Left. }Illustration of $\bSigma=\bM_1+\cdots+\bM_4$ where $\bM_k$ is labeled by $k$, for $k=1,2,3,4$. $I_j$ is unknown except for the fact that it contains $i_1$ and $i_2$. In step 1 and on input $(\bSigma,(i_1,i_2))$, $\subalgname$ searches for a set $L$ such that ${\bSigma(\{i_1,i_2\}\times L)}$ (maroon) is a rank one, fully dense, submatrix of $\bSigma$. {\bf Right.} Step 2 recovers $\ba_j$ from the submatrix $\bSigma(\cdot\times L)$ (between the dotted lines), using that $\bSigma(\cdot\times L)$ agrees with $\bM_j(\cdot\times L)$ except for sparse errors (dark gray).}
\end{figure}

\begin{algorithm}[H]
  \floatname{algorithm}{Procedure}
  \label{alg:subalg}
  \caption{Population \SUBALGNAME}
  \begin{algorithmic}[1]
    \STATE {\bf input: }$\bSigma=(\gamma_1\ldots\gamma_r)^\top$, $(i_1,i_2)$
    \STATE $K\assign \supp\gamma_{i_1}\cap\supp\gamma_{i_2}$
\STATE $\hat\varphi\assign\mode(\gamma_{i_2}(K)/\gamma_{i_1}(K))$
\STATE $L\assign \{k\in K\:|\:\gamma_{i_2}(k)/\gamma_{i_1}(k)=\hat\varphi\}$
\FOR{$i=1,\ldots,r$}\label{loop}
\STATE $\tilde\ba(i)\assign \operatorname{median}(\gamma_i(L)/\gamma_{i_1}(L))$
\ENDFOR
    \STATE {\bf output: }$\hat\ba=\sqrt{\tfrac{\Sigma_{i_1i_2}}{\tilde\ba(i_2)}}\tilde\ba$ 
  \end{algorithmic}
\end{algorithm}
\paragraph{Deflation.}
The algorithm to construct $\hat\bA$ on input $\bSigma$ works by repeatedly applying $\subalgname$ and a deflation step replacing $\bSigma$ with $\bSigma-\hat\ba\hat\ba^\top$. First, initialize $\hat\bA$ as a width 0 matrix. At the start of each iteration, pick a nonzero entry $\Sigma_{i_1i_2}$, and assign $\hat\ba\assign\subalgname(\bSigma,(i_1,i_2))$. Append $\hat\ba$ to $\hat\bA$, and subtract $\hat\ba\hat\ba^\top$ from $\bSigma$. Repeat the above procedure until $\bSigma$ is diagonal, and output $\hat\bA$.
A potential problem with the deflation procedure is that an incorrect output from $\subalgname$ column could impede the subsequent applications of $\subalgname$. And indeed, $\subalgname$ fails for the (atypical) pairs $i_1,i_2$ such that $|R_{i_1}\cap R_{i_2}|\geq 2$. To overcome the problem of an incorrect output of $\subalgname$ we let $I=\supp \hat\ba$ and verify that $\bSigma(I\times I)-\hat\ba(I)\hat\ba(I)^\top$ is sparser than $\bSigma(I\times I)$. This ensures that $\hat\ba$ is correct. Only then do we $\hat\ba$ append to $\hat\bA$ and deflate $\bSigma$ by subtracting $\hat\ba\hat\ba^\top$.

\section{Finite sample case}

When the input to $\subalgname$ is an approximation $\overline\bSigma=(\overline\gamma_1,\ldots,\overline\gamma_r)^\top$ to $\bSigma$ we have to relax the requirement that $\gamma_{i_1}(L)$ and $\gamma_{i_2}(L)$ be collinear. The statement that the two vectors $\gamma_{i_1}(L)$ and $\gamma_{i_2}(L)$ in $\R^L$ are collinear is equivalent with saying that the points $(\gamma_{i_1}(k),\gamma_{i_1}(k))\in R^2$ are collinear, where $k$ ranges over $L$. We take this as the starting point of the relaxed definition of $L$. First, to bound the noise/signal ratio we pick a small constant $c>0$ and disregard points within a distance $c$ to a coordinate axis in $\R^2$. This means that $K=\supp\gamma_{i_1}\cap\supp\gamma_{i_2}$ is replaced by $K=\{k\in[r]:|\ol\gamma_{i_1}(k)|\wedge|\ol\gamma_{i_2}(k)|\geq c\}$. We then approximate the entries of $\overline\gamma_{i_2}/\overline\gamma_{i_1}$ by values in a discrete set $Z\subset \R$. This corresponds to placing the points $(\overline\gamma_1(k),\overline\gamma_2(k))$ into bins where each bin is a cone in $\R^2$.
The appropriate choice of discretization is $Z=\pm e^{\varepsilon\bbZ}$ where
\[\pm e^{\varepsilon\bbZ}=\big\{\varphi\in\R\backslash\{0\}\:\big|\:\tfrac1\varepsilon\log|\varphi|\in\bbZ\big\}.\]
In particular, the set of bins arising from this discretization is symmetric about the axis ${\{(x,x)|x\in\R\}}$ because $y/x=\pm e^{\varepsilon m}$ implies $x/y=\pm e^{-\varepsilon m}\in \pm e^{\varepsilon\Z}$.

\begin{definition}\label{def:L_def}
  For vectors $\gamma_1,\gamma_2\in\R^K$, $\varphi\in\R\setminus\{0\}$, and $\varepsilon>0$, define
  \[L_{\varphi,\varepsilon}(\gamma_1,\gamma_2)=\Big\{k\in K\Big|\:\:e^{-\varepsilon}\leq\varphi^{-1}\frac{\gamma_2(k)}{\gamma_1(k)}< e^\varepsilon\Big\}.\]
\end{definition}
Pick $\varepsilon>0$ such that $c\varepsilon/4>|\overline\bSigma-\bSigma|_\infty$. Such $\varepsilon$ can be estimated as from $\ol\bSigma$ using the fact that a large fraction of the entries of $\bSigma$ are zero.

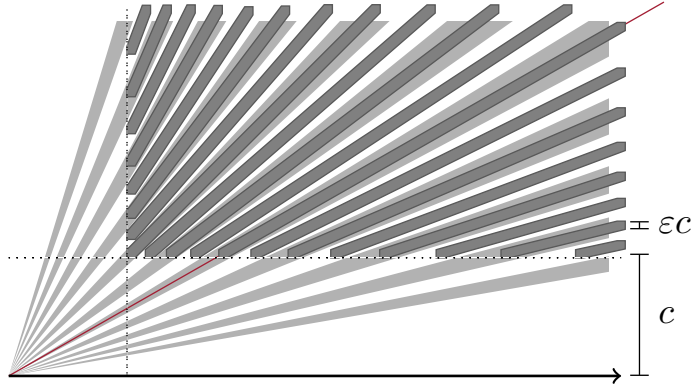
\begin{figure}
  \centering
  \resizebox{0.8\columnwidth}{!}{

\definecolor{maroon}{RGB}{163,31,52}
\definecolor{grey}{gray}{0.7}
\definecolor{darkgray}{gray}{0.35}
\newcommand*{\threshold}{0.01}
\newcommand*{\epsi}{0.07}
\newcommand*{\hgt}{0.03}
\newcommand*{\wdh}{0.05}
\newcommand*{\minwdh}{0.0}
\pgfmathsetmacro{\epsscaled}{\epsi*\threshold/2}

\begin{tikzpicture}[scale=6.5/\wdh]

\begin{scope}
\clip (\minwdh,0) rectangle (\wdh-0.001,\hgt-0.001);

  \foreach \k in {-16,-12,...,24}{
  \pgfmathsetmacro{\a}{{exp(-(\k+1)*\epsi)}};
  \pgfmathsetmacro{\actualwdh}{min(\wdh,\hgt/\a)};
\fill[color=grey] (0,0)--(\actualwdh,{\a*\actualwdh})--(\actualwdh,{exp(-(\k-1)*\epsi)*\actualwdh})--cycle;  
  }
\end{scope}

\draw[color=maroon](0,0)--(1.07*\wdh,{1.07*\wdh*exp(-8*\epsi)});

\foreach \k in {-14,-12,...,22}{
  \pgfmathsetmacro{\a}{{exp(-\k*\epsi)}};
  \pgfmathsetmacro{\actualwdh}{min(\wdh,\hgt/\a)};
  \pgfmathsetmacro{\actualminwdh}{max(\threshold,\threshold/\a)};
  
  \draw[color=darkgray,fill=gray] (\actualminwdh-\epsscaled,\a*\actualminwdh+\epsscaled)--(\actualminwdh-\epsscaled,\a*\actualminwdh-\epsscaled)--(\actualminwdh+\epsscaled,\a*\actualminwdh-\epsscaled)--(\actualwdh+\epsscaled,\a*\actualwdh-\epsscaled)--(\actualwdh+\epsscaled,\a*\actualwdh+\epsscaled)--(\actualwdh-\epsscaled,\a*\actualwdh+\epsscaled)--cycle;
}

\draw[dotted,fill=white] (\minwdh,\threshold-\epsscaled) rectangle (\wdh,\threshold-\epsscaled);
\draw[dotted,fill=white] (\threshold-\epsscaled,0) rectangle (\threshold-\epsscaled,\hgt);
\draw[thick,->](\minwdh,0)--(\wdh,0);
\draw[|-|](1.03*\wdh,0)--(1.03*\wdh,\threshold);
\node[right] at (1.04*\wdh,\threshold/2){$c$};
\draw[|-|,left](1.03*\wdh,1.2*\threshold)--(1.03*\wdh,1.2*\threshold+2*\epsscaled);
\node[right] at (1.04*\wdh,1.2*\threshold+\epsscaled){$\varepsilon c$};

\end{tikzpicture}

  }
  \caption{In the population case, the points $\{(\gamma_{i_1}(k),\gamma_{i_2}(k))\:|\:k\in L\}$ lie on a line through the origin (maroon), say, with slope $\varphi$. If $(\gamma_{i_1}(k),\gamma_{i_2}(k))$ is bounded away from the coordinate axes by at least $c$, then the corresponding approximation $(\ol\gamma_{i_1}(k),\ol\gamma_{i_2}(k))$ is in a bar-shaped set (dark gray). Each bar-shaped set is contained in a cone-shaped one (light gray), which implies that $e^{-\varepsilon}\leq\varphi^{-1}{\ol\gamma_{i_2}(k)}/{\ol\gamma_{i_1}(k)}< e^\varepsilon$.}
\label{anglebarfigure}
\end{figure}

\begin{algorithm}
  \floatname{algorithm}{Procedure}
  \label{alg:noisysubalg}
  \caption{\SUBALGNAME}
  \begin{algorithmic}[1]
    \STATE {\bf input: }$\ol\bSigma=(\ol\gamma_1\ldots\ol\gamma_r)^\top$, $(i_1,i_2)$, $c>\varepsilon>0$
    \STATE $K\assign\{k\in[r]\::\:|\ol\gamma_{i_1}(k)|\wedge|\ol\gamma_{i_2}(k)|\geq c\}$
\STATE $\hat\varphi\assign\argmax_{\varphi\in\pm e^{\varepsilon\Z}}|L_{\varphi,\varepsilon}(\ol\gamma_{i_1}(K),\ol\gamma_{i_2}(K))|$
\STATE $L\assign L_{\hat\varphi,2\varepsilon}(\ol\gamma_{i_1}(K),\ol\gamma_{i_2}(K))$

\FOR{$i=1,\ldots,r$}
\STATE $\tilde\ba(i)\assign \operatorname{median}(\ol\gamma_i(L)/\ol\gamma_{i_1}(L))$
\ENDFOR
    \STATE {\bf output: }$\hat\ba=\sqrt{\tfrac{\ol\Sigma_{i_1i_2}}{\tilde\ba(i_2)}}\tilde\ba$ 
  \end{algorithmic}
\end{algorithm}
The $\argmax$ over $\varphi\in{\pm e^{\varepsilon\bbZ}}$ can be computed by computing two modes. First divide $K$ into two sets $K_+$ and $K_-$ according to the sign of $\ol\gamma_{i_2}(k)/\ol\gamma_{i_1}(k)$. Then for $K'=K_+,K_-$ make a list of $|K'|$ integers $\lfloor\frac1\varepsilon\log|\ol\gamma_{i_2}(k)/\ol\gamma_{i_2}(k)|\rfloor$ and take the mode, pairing neighboring integers (for example by appending a copy of the list with $1$ added to each integer). This yields a mode for each set $K_+,K_-$, and we finish by taking the most frequent of the two.



\section{Structure of the mixing matrix}
The proof of theorem \ref{thm:sampletheorem} uses the fact that $\bA\sim\BG(r,s,\theta)$ satisfies a condition \OC{h}, which we define below, and which captures the generic structure of $\bA$.

\begin{definition}\label{singleoverlap}
For indices $i,j$ let
\[C_{j,i}=\bigcup_{k\in R_i\setminus\{j\}} I_j\cap I_k\setminus \{i\},\]
and define $m_j=\max_{i=1,\ldots,r}|C_{j,i}|$.
\end{definition}
A graphical illustration of the sets $C_{i,j}$ is given in Figure~\ref{Cji_figure}.  The following \emph{overlap condition} (\OC{h}) is required for our recovery guarantee.

\begin{condition}[\OC{h}]\label{OC}
 Let $h\geq0$ and $\essenS_j=\{i:|\ba_j(i)|\geq\frac{1}{10}\}\subset I_j$. We say that $\bA$ satisfies the overlap condition \OC{h} if for every $j$,
  \[6m_j+2|I_j\setminus\essenS_j|+h<|I_j|.\]
\end{condition}
\begin{figure}
\label{FIG:singleoverlap}
\centering
\resizebox{\columnwidth}{!}{
\definecolor{maroon}{RGB}{163,31,52}
\definecolor{grey}{RGB}{230,230,230}

\def\r{7}
\def\s{7}
\def\list{
2/1,4/1,5/1,
3/2,1/2,7/2,
3/3,6/3,
1/4,2/4,6/4,
2/5,4/5,7/5,
1/6,3/6,5/6,
2/7,6/7,6/1}

\def\Sj{2,4,5,6}
\def\Ri{2,3,6}

\def\redlist{6/3,5/6}
\def\Cji{5,6}

\def\i0{3}
\def\j0{1}


\begin{tikzpicture}

\begin{scope}[yscale=-0.45,xscale=0.45]
\foreach \i in \Sj
	\fill(\j0-0.4,\i-0.6) rectangle (\j0+0.4,\i+0.6);
\foreach \i in \Sj
	\fill[color=grey](\j0-0.3,\i-0.5) rectangle (\j0+0.3,\i+0.5);

\draw[color=grey](0.5,0.5)rectangle(\r+0.5,\s+0.5);
	
\draw[dotted] (0.5,\i0)--(\s+0.5,\i0);
\draw[dotted] (\j0,0.5)--(\j0,\r+0.5);
\foreach \i in \Sj
	\fill[opacity=0.1](0.5,\i-0.5) rectangle (\s+0.5,\i+0.5);
\foreach \j in \Ri
	\fill[opacity=0.1](\j-0.5,0.5) rectangle (\j+0.5,\r+0.5);

\foreach \i in \Sj	
	\foreach \j in \Ri
		\fill[opacity=0.28](\j-0.5,\i-0.5) rectangle (\j+0.5,\i+0.5);

\foreach \i/\j in \list
	\fill[color=black](\j,\i) circle (0.1);

\node at (\j0,0){j=1};
\node at (-0.2,\i0){i=3};
\end{scope}

\begin{scope}[yscale=-0.45,xscale=0.45,xshift=14cm]
	
\draw[color=grey](0.5,0.5)rectangle(\r+0.5,\s+0.5);

\draw[color=white] (-0.5,0)--(\s+1.5,\r);

\foreach \i in \Sj
	\fill(\j0-0.4,\i-0.6) rectangle (\j0+0.4,\i+0.6);
\foreach \i in \Sj
	\fill[color=grey](\j0-0.3,\i-0.5) rectangle (\j0+0.3,\i+0.5);
	
\foreach \i in \Sj	
	\foreach \j in \Ri
		\fill[opacity=0.28](\j-0.5,\i-0.5) rectangle (\j+0.5,\i+0.5);	
	
\foreach \i in \Cji{
	\fill
	[color=maroon](\j0-0.25,\i-0.45) rectangle (\j0+0.25,\i+0.45);}

\foreach \i/\j in \list
	\fill[opacity=0.2](\j,\i) circle (0.1);
\foreach \i/\j in \redlist{
	\fill(\j,\i) circle (0.1);
	\draw[->,thick,color=maroon](\j,\i)--(1.6,\i);
	\fill(\j,\i) circle (0.1);
	}

\node at (\j0,0){j};
\node at (0,\i0){i};

\end{scope}
\end{tikzpicture}

}
\caption{Illustration of $C_{j,i}\subset I_j$ for $i=3$ and $j=1$. The sparsity pattern of $\bA$ is shown with black dots, and $I_1$ is indicated with a black border. $C_{1,3}$ (maroon) is defined as the union of the column supports of the submatrix $\bA_{I_1\times R_3}$ shown in dark gray.}
\label{Cji_figure}
\end{figure}
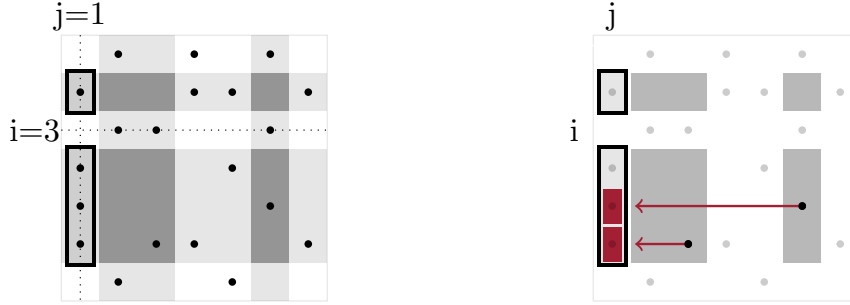
\OC{h} is a condition on $\bA$ which ensures that each square $I_j\times I_j$ has a small overlap with $\bigcup_{k\neq j}I_k\times I_k$. This holds in the regime $\theta\ll 1/\sqrt s$ where $\bSigma=\sum_k\ba_k\ba_k^\top$ is sparse, which is clear by considering it as a sparsity condition on $\sum_{k\neq j}\ba_k(I_j)\ba_k^\top(I_j)$. More precisely we have the following theorem.
\begin{theorem}\label{thm:OC_theorem}
Let $\bA\sim\BG(r,s,\theta)$. There exists a choice of constants $C,c>0$ such that if
\begin{equation}\label{eq:OC_requirement}
C\frac{\log(r/\delta)}{r}
\le \theta \le 
\frac{c}{\sqrt{s}+\log(r/\delta)}\,,
\end{equation}
then $\bA$ satisfies condition \OC{\frac{r\theta}{10}} with probability at least $1-\delta$, 
\end{theorem}
This theorem requires implicitly that $r/\log(r)\ge C  \sqrt{s}$.  In essence,~\eqref{eq:OC_requirement} enforces that, with high probability, each row of $\bA$ has at most $\sqrt{s}$ non-zero entries.

The subprocedure $\subalgname$ identifies a large subset $I\subset I_j$ by searching for $I$ of size $h=r\theta/10$ such that $\bSigma((i_1,i_2)\times I)$ is nearly singular. If such approximately singular $2$-by-$h$ submatrices appear by chance in $\bSigma$, then this step fails. We therefore define $h_\varepsilon(\bA)$ as the maximum width such that this occurs. As the approximation $\overline\bSigma$ gets better, i.e. as $\varepsilon$ decreases, our notion of approximately singular becomes more restrictive. This implies that $h_\varepsilon(\bA)$ is an increasing function of $\varepsilon$ (decreasing in the accuracy $1/\varepsilon$).

\begin{definition}\label{def:h_definition}
   For $\bA=(\ba_1,\ldots,\ba_s)=(\rho_1,\ldots,\rho_r)^\top$ with $I_j=\supp\ba_j$ and $\varepsilon\geq0$, define
  \[h_\varepsilon(\bA)=\max_j\max_{(i_1,i_2)\in I_j}\sup_{\varphi\neq 0}|L_{\varphi,\varepsilon}(\bA\rho_{i_1},\bA\rho_{i_2})\backslash I_{j}|\text{.}\]
  \end{definition}
Akin to the restricted isometry property pervasive to compressed sensing, the randomly generated matrix $\bA$ satisfies condition \OC{h_\varepsilon(\bA)} with high probability.

\begin{theorem}\label{thm:DC_theorem}
  Let $\bA\sim\BG(r,s,\theta)$. There exists a choice of constant $c>0$ such that if \eqref{eq:OC_requirement} holds and
\begin{equation}\label{epsilonchoice}{\varepsilon=\frac{c}{s^2\theta^3+\log(r/\delta)}},\end{equation}
then $h_\varepsilon(\bA)\leq\frac{r\theta}{10}$ with probability $1-\delta$.
\end{theorem}
Theorems \ref{thm:OC_theorem} and \ref{thm:DC_theorem} are proven in appendix~\ref{sec:structure_theorem_section}.


\section{Proof of recovery}\label{guarantee_section}
We first show that our algorithm recovers $\bA$ when the population covariance matrix $\bSigma=\bA\bA^\top$ is known exactly. In this case we can set $\varepsilon=0$, so we benefit from the fact that $h_0(\bA)=1$.
\begin{theorem}\label{thm:correctness}
  Let $\bA$ be such that \OC{h_\varepsilon(\bA)} holds. Then our algorithm halts in $T$ iterations of $\subalgname$ and outputs $\hat\bA$ such that $d(\hat\bA,\bA)$, where $\bbE T\leq\frac65s$, and the expectation is over the randomness of the algorithm.
\end{theorem}
\begin{proof}
  Let
  \begin{equation}\label{Sigma_j}\bSigma_j=I_j\times I_j\setminus\bigcup_{l\neq j}\Big(I_l\times I_l\Big),\end{equation}
  and note that $(i_1,i_2)\in\bSigma_j$ iff $R_{i_1}\cap R_{i_2}=\{j\}$. We first show that for $i_1\neq i_2$, such that $(i_1,i_2)\in \bSigma_j$, it holds that $\subalgname(\bSigma,(i_1,i_2))=\ba_j$.
  
  \paragraph{Correctness of \SUBALGNAME.}
   Let $\varphi=A_{i_2j}/A_{i_1j}$. We then have that
  \begin{equation}\label{eq:inclusion}I_j\backslash (C_{j,i_1}\cup C_{j,i_2})\subset L_{\varphi}(\gamma_{i_1},\gamma_{i_2}).\end{equation}
  Indeed, for all $k\in I_j\backslash(C_{j,i_1}\cup C_{j,i_2})$ it holds that $R_{i_1}\cap R_k=R_{i_2}\cap R_k=\{j\}$, hence $\Sigma_{i_1k}=\ba_j(i_1)\ba_j(k)$ and $\Sigma_{i_2k}=\ba_j(i_2)\ba_j(k)$. It follows that $\gamma_{i_2}(k)/\gamma_{i_1}(k)=\Sigma_{i_2k}/\Sigma_{i_1k}=\ba_j(i_2)/\ba_j(i_1)=\varphi$, i.e. $k\in L_{\varphi}(\gamma_{i_1},\gamma_{i_2})$. This shows that $I_j\backslash(C_{j,i_1}\cup C_{j,i_2})\subset L_{\varphi}(\gamma_{i_1},\gamma_{i_2})$.

  By \eqref{eq:inclusion}, it holds that
  \begin{align*}|L|=&\max_{\tilde\varphi\in\pm e^{\varepsilon\bbZ}}|L_{\tilde\varphi,\varepsilon}(\gamma_{i_1},\gamma_{i_2})|\\\geq& |I_j\backslash(C_{j,i_1}\cup C_{j,i_2})|\geq |I_j|-2m_j.\end{align*}
  Write $\hat\varphi\sim\varphi$ if $e^{-\varepsilon}\leq\hat\varphi/\varphi\leq e^{\varepsilon}$. For every $\tilde\varphi\not\sim\varphi$ we have that $L_{\varphi,\varepsilon}(\gamma_{i_1},\gamma_{i_2})\leq h_\varepsilon(\bA)<|I_j|-2m_j$ by \OC{h_\varepsilon(\bA)}. Hence $\hat\varphi\sim\varphi$, and $L_\varphi(\gamma_{i_1},\gamma_{i_2})\subset L$. Moreover, $|L\backslash I_j|\leq h_\varepsilon(\bA)$.

  Consider the loop over $i\notin\{i_1,i_2\}$. For every $k\in I_j\backslash(C_{j,i_1}\cup C_{j,i_2}\cup C_{j,i})\subset L$ it holds that $\gamma_i(k)/\gamma_{i_1}(k)=\ba_j(i)/\ba_j(i_1)$. There are at least $|I_j\backslash(C_{j,i_1}\cup C_{j,i_2}\cup C_{j,i})|\geq|I_j|-3m_j$ such $k$'s, out of at most $|I_j|+h_\varepsilon(\bA)$ entries of $\gamma_i(L)/\gamma_{i_1}(L)$. Hence, more than half the entries of $\gamma_i(L)/\gamma_{i_1}(L)$ take the value $\ba_j(i)/\ba_j(i_1)$, provided that
  \[|I_j|-3m_j>\frac{|I_j|+h_\varepsilon(\bA)}2,\]
  which is equivalent with
  \[6m_j+h_\varepsilon(\bA)<|I_j|.\]
  This holds by \OC{h_\varepsilon(\bA)}, and we conclude that $\tilde\ba(i)=\operatorname{median}(\gamma_i(L)/\gamma_{i_1}(L))=\ba_j(i)/\ba_j(i_1)$. Now $\Sigma_{i_1i_2}=\ba_j(i_1)\ba_j(i_2)$ and $\tilde\ba_j(i_2)=\ba_j(i_2)/\ba_j(i_1)$, so $\sqrt{\Sigma_{i_1i_2}/\tilde\ba(i_2)}=|\ba_j(i_1)|$ which implies $\hat\ba=\sqrt{\Sigma_{i_1i_2}/\tilde\ba(i_2)}\tilde\ba_j=\sigma\ba_j$ where $\sigma=\sgn\ba_j(i_1)$. Hence $\ba_j$ was recovered correctly.

  \paragraph{Correctness of deflation method.}
 We now show that each column is recovered correctly by applying $\subalgname$ repeatedly.
 We justify below that there is a (random) sequence of sets $J_t\subset\{1,\ldots,r\}$,
 \[\emptyset=J_0\subset J_1\subset J_2\subset\ldots,\]
 such that after the $t$th iteration, $\bSigma^{(t)}=\sum_{j\notin J_t}\ba_j\ba_j^\top$ and $\hat\bA$ is a concatenation (in some order) of the columns $\{\pm\ba_j|j\in J_t\}$. If we ever have $|J_t|=s$, then $\bSigma^{(t)}=0$, so the algorithm halts. We have shown in the paragraph above that $\ba_j$ is correctly recovered if $(i_1,i_2)\in \bSigma_j$. Therefore, $J_{t+1}>J_t$ if $(i_1,i_2)\in\bigcup_{j\notin J_t}\bSigma_j$. Since $(i_1,i_2)$ is chosen uniformly among the at most $\sum_{j\notin J_t}|I_j|^2$ nonzero off-diagonal entries of $\bSigma^{(t)}=\sum_{j\notin J_t}\ba_j\ba_j^\top$, we have for $J\subsetneq \{1,\ldots,r\}$,
 \begin{equation}\label{eq:constantprob}\prob(J_{t+1}\supsetneq J_t|J_t=J)\geq\frac{\sum_{j\notin J}|\bSigma_j|}{\sum_{j\notin J}|I_j|^2}\geq \frac56\end{equation}
 for each $t\in\N$. Here we have used the properties that the $\bSigma_j$ are disjoint, and that $|\bSigma_j|\geq \frac56|I_j|^2$ by \OC{h}. Let $T=\min\{t|\:|J_t|=r\}$, and define
 \[\cS_t=|J_t|-\frac56(t\wedge T).\]
 Note that $\cS_t=s-\frac56T$ for all $t\geq T$. \eqref{eq:constantprob} says that for $t<T$, $J_{t+1}$ is larger than $J_t$ with probability at least $5/6$. This implies that $\cS_t$ is a submartingale. Hence,
 \[0=\cS_0\leq\bbE \cS_T=s-\frac56\bbE T,\]
 which gives us the bound $\bbE T\geq \frac65s$.

It remains to show that at each iteration there is a $J$ such that $\bSigma=\sum_{j\notin J}\ba_j\ba_j^\top$ and $\hat\bA$ is a concatenation (in some order) of the columns $\{\pm\ba_j|j\in J\}$. This holds before the first iteration with $J=\emptyset$, and the property is preserved when a column is correctly recovered. We have shown that this happens whenever $|R_{i_1}\cap R_{i_2}|=1$. If instead $|R_{i_1}\cap R_{i_2}|\geq 2$, then the scaling factor $\sqrt{\Sigma_{i_1i_2}/\tilde\ba(i_2)}$ does not coincide with either $\pm\ba_j(i_1)$ for any $j\in R_{i_1}\cap R_{i_2}$, hence no cancellation occurs when subtracting $\hat\ba\hat\ba^\top$ from $\bSigma$. Therefore $\bSigma$ is not changed in the current iteration, and the property still holds.
\end{proof}
\begin{figure}
  \begin{center}
 \resizebox{0.5\columnwidth}{!}{

\definecolor{maroon}{RGB}{163,31,52}
\definecolor{grey}{RGB}{230,230,230}
\newcommand{\drawamatrix}[3]{\draw [#3](#1,10-#1) rectangle (#2,10-#2);
}
\newcommand{\fillamatrix}[3]{\fill [#3] (#1,10-#1) rectangle (#2,10-#2);
}
\newcommand{\filladiscontinuous}[5]{
\fill [#5] (#1,10-#1) rectangle (#2,10-#2);
\fill [#5] (#3,10-#3) rectangle (#4,10-#4);

\fill [#5] (#1,10-#3) rectangle (#2,10-#4);
\fill [#5] (#3,10-#1) rectangle (#4,10-#2);
}

\begin{tikzpicture}
\begin{scope}[scale=0.7]
\drawamatrix{2}{7}{fill=grey};

\fillamatrix{1.2}{3.2}{color=maroon}
\drawamatrix{2}{7}{};

\filladiscontinuous{0}{1.5}{2.5}{3.5}{opacity=0.28}
\filladiscontinuous{5.7}{6.8}{7.5}{8.5}{color=maroon}
\fillamatrix{6.3}{9}{opacity=0.4}

\draw [ thick] (3.5,10-3)--(7,10-3)node [midway,below]{$I_j\backslash C_{j,i_1}$};
\draw [ thick] (2,10-6.7)--(5.7,10-6.7);

\draw [arrows={-angle 90},dotted] (3.5,10-3)--(3.5,0.5);
\draw [arrows={-angle 90},dotted] (5.7,10-6.7)--(5.7,0.5);

\draw [thick,|-|,dotted,color=maroon] (2,0)--(3.5,0) node [midway,below]{$\leq m_j$};
\draw [thick,|-|,dotted,color=maroon] (5.7,0)--(7,0) node [midway,below]{$\leq m_j$};

\draw [thick,|-|] (3.5,0)--(5.7,0) node [midway,above]{$L$};

\node at (-0.5,10-3){$i_1$};
\node at (-0.5,10-6.7){$i_2$};
\end{scope}
\end{tikzpicture}

}
\end{center}
\caption{Illustration of $\bSigma$. The shaded areas represent the sets $I_k\times I_k$, $k=1,\ldots,s$. The largest square is $I_j\times I_j$. The set $I_j\setminus C_{j,i}$ represents the set of entries in the $i$th row of $\bSigma$ which have a contribution from $\ba_j\ba_j^\top$ and none from $\ba_k\ba_k^\top$ for $k\neq j$. We have that $|L|\geq |I_j\backslash (C_{j,i_1}\cup C_{j,i_2})|\geq |I_j|-2m_j$.}
\end{figure}

We now treat the robustness of algorithm to uniformly small errors, i.e. we consider an input $\overline\bSigma$ with a bound on $|\overline\bSigma-\bSigma|_\infty$. The following lemma and theorem express the robustness of $L_{\varphi,\varepsilon}(\gamma_1(K),\gamma_2(K))$ to perturbations. 

\begin{lemma}\label{lem:noiselemma}
  Let $0<\varepsilon<1$. Let $\gamma_1,\gamma_2\in(\R\setminus[-c,c])^K$ , and let $\overline\gamma_i\in\R^K$, $i=1,2$, be such that $|\overline\gamma_i-\gamma_i|_\infty\leq (1-e^{-\varepsilon/2})c$. Then,
\[L_{\varphi,0}(\gamma_1,\gamma_2)\subseteq L_{\varphi,\varepsilon}(\overline\gamma_1,\overline\gamma_2)\text{.}\]
\end{lemma}

From lemma \ref{lem:noiselemma} it follows that the columns of $\bA$ are approximately recovered by $\subalgname$.
\begin{theorem}\label{thm:one_column}
  Let $\bA\in\R^{r\times s}$ satisfiy \OC{h_{2\varepsilon}(\bA)}. Let $\overline\bSigma$ be such that $|\overline\bSigma-\bA\bA^\top|_\infty\leq (1-e^{-\varepsilon/4})c$, where $c$ is as in definition \ref{def:L_def}. Let $(i_1,i_2)\in\bSigma_j$, and suppose $|\ba_j(i_1)|\wedge|\ba_j(i_2)|\geq 10c$. Then ${\hat\ba=\subalgname(\overline\bSigma,(i_1,i_2),c,\varepsilon)}$ satisfies
  \[\min_{\sigma=\pm1}|\sigma\hat \ba-\ba_j|_\infty=O(|\bA|_\infty|\ba_j|_\infty\varepsilon).\]
\end{theorem}
Lemma \ref{lem:noiselemma} and theorem~\ref{thm:one_column} are proven in appendix \ref{sec:robust}.
\medskip

  \paragraph{Recovery from empirical covariance matrix.}
In the finite sample case we apply theorem \ref{thm:correctness} using an entrywise uniform bound on $\overline\bSigma-\bSigma$, where $\bSigma\sim\cW(\bSigma,n)$ is the empirical covariance matrix.
A standard computation shows that with probability at least $1-\delta$, the bound $|\hat\bSigma-\bSigma|_\infty=O(|\bSigma|_\infty\sqrt{\log(r/\delta)/n})$ holds.
\newtheorem{maintheorem}{Theorem \ref{thm:sampletheorem}}
\renewcommand\themaintheorem{\unskip}
\begin{maintheorem}\mnthm\end{maintheorem}
\begin{proof}
  By theorems \ref{thm:OC_theorem} and \ref{thm:DC_theorem} and a union bound we have that with probability $1-2\delta$, \OC{r\theta/10} holds and $h_\varepsilon(\bA)\leq r\theta/10$. These two properties imply \OC{h_\varepsilon(\bA)}.

  We pick \[n=C\frac{|\bSigma|_\infty^2\log(r/\tilde\delta)}{\varepsilon^2}.\]
  By corollary \ref{thm:covariance_concentration},
  \[|\overline\bSigma-\bSigma|_\infty\leq 6|\bSigma|_\infty\sqrt{\frac{\log(r/\tilde\delta)}{n}}= c\varepsilon\]
for appropriate choice of $C$. The conditions of theorem \ref{thm:correctness} are therefore met, and we get
  \[
    d(\hat\bA,\bA)=O(|\bA|_\infty^2\varepsilon)=O\bigg(|\bA|_\infty^2|\bSigma|_\infty\sqrt{\frac{\log(r/\tilde\delta)}{n}}\bigg).
  \]
  To bound the running time of our algorithm, note that the running time of $\subalgname$ is dominated by the $r$ assignments $\tilde\ba(i)\assign\operatorname{median}(\gamma_{i}(L)/\gamma_{i_1}(L))$, each of which takes time $O(L)$. Hence an application of $\subalgname$ takes time $O(rL)=O(r^2\theta)$. As previously shown, the expected number of applications of $\subalgname$ is bounded by $\frac65s$. So the expected running time is $O(r^2s\theta)$. 
\end{proof}
 The error term in theorem \ref{thm:sampletheorem} depends on $|\bSigma|_\infty=\max_{i=1,\ldots,r}\|\rho_i\|_2^2$, which is of the order $s\theta$, the expected support size of a row of $\bA$. We couple the parameters in the preceding theorem to get the asymptotic result in corollary \ref{cor:simplifiedmain}.
\begin{proof}[Proof of corollary \ref{cor:simplifiedmain}]
  The upper and lower bounds on $\theta$ in the conditions of theorem \ref{thm:sampletheorem} are satisfied for large values of $s$. By corollary \ref{cor:AAbound}, $\prob(|\bA\bA^\top|_\infty>5s\theta)\leq re^{-s\theta/6}=o(1)$. Hence, the choice of $n$ in theorem \ref{thm:sampletheorem} satisfies
  \[\log n\sim \log(s^6\theta^8)\vee\log(s^2\theta^2)\sim ((6-8\alpha)\vee(2-2\alpha))\log s,\]
  where $LHS\sim RHS$ means $LHS/RHS\to1$. 
  We can choose $n$ larger than in theorem \ref{thm:sampletheorem} by setting $\beta=\frac12\log_s n>(3-4\alpha)\vee(1-\alpha)$. Applying theorem \ref{thm:sampletheorem} and the bound $|\bA\bA^\top|_\infty=O(s\theta)$ yields $d(\hat\bA,\bA)=\tilde O(|\bA|_\infty^2s\theta/\sqrt n)=O(|\bA|_\infty^2s^{1-\alpha-\beta})$, where $\tilde O$ hides a logarithmic factor. We conclude using the bound $|\bA|_\infty=O(\sqrt{\log(s/\delta)})$ (It is the maximum absolute value of $<cs^2$ Gaussians) which holds with probability $1-\delta$ and choosing for example $\delta=1/s$. 
  \end{proof}

%


\section{Conclusion and future directions}
We have given the first rigorous treatment of a model of ICA which replaces distributional assunmptions on the sources with structural assumptions on the mixing matrix. We have assumed a sparse mixing matrix, a setting which has previously been explored in the experimental literature. In contrast with previous work in this direction which adds a penalty term to the optimization problem solved by traditional ICA, we have given an entirely different combinatorial algorithm and proven that it recovers a sparse and generic mixing matrix from only the second moments of the observations. In particular, our algorithm works even in the setting of Gaussian sources where other ICA methods fail.
Our method requires the sparsity parameter to be at most $1/\sqrt{s}$ in order to yield a covariance with a constant $c<1$ fraction nonzero entries. However, the fact that specifying an $r$-by-$r$ covariance matrix takes $r(r+1)/2$ parameters suggests that a mixing matrix with a constant fraction nonzero entries may be identifiable from second moments. It remains an open problem to formulate weaker assumptions on the mixing matrix under which it can be estimated from second moments.

\newpage
\appendix
\section{Proofs of structural properties}
\label{sec:structure_theorem_section}

Condition \OC{h} bounds the size of $C_{j,i}$ relative to $|I_j|$. Since $C_{j,i}$ is defined in terms of the submatrix $A_{I_j\setminus\{i\}\times R_i\setminus\{j\}}$, the first step in proving this bound is to condition on $R_i$ and $I_j$ as in the following lemma.

\begin{lemma}\label{submatrixlemma}
Let $\bA\sim \BG(r,s,\theta)$. For a fixed pair $i,j$, conditioned on the column and row supports $I_j=\supp \ba_j$ and $R_i=\supp \rho_i$, it holds with probability at least $1-\delta$ that
\[|C_{j,i}|\leq |I_j|\:|R_i|\theta+\sqrt{6|I_j|\:|R_i|\theta\log(1/\delta)}\]
which implies in particular
\[|C_{j,i}|\leq \frac54|I_j|\:|R_i|\theta+6\log(1/\delta)\]
\end{lemma}
\begin{proof}
  Write $I'=I_j\setminus\{i\}$ and $R'=R_i\setminus \{j\}$, and introduce the submatrix $\bA'=A_{I'\times R'}$ with rows $\rho'_k\in\R^{|R'|}$, $k\in I'$. $C_{j,i}$ is a random subset of $I'$, and the events ${E_{k}=\{k\in C_{j,i}\}=\{\rho'_k\not\equiv 0\}}$, $k\in I'$ are independent with
  \[p=\prob(E_k)=1-(1-\theta)^{|R'|}\leq |R'|\theta\]
The conclusion follows by applying \eqref{eq:Okamotologeq} and corollary \ref{cor:weakbound} \ref{concentration} to ${|C_{j,i}|\sim\Bin(|I'|,p)}$.
\end{proof}

\begin{lemma}\label{lem:overlapbound}
Let $\bA\sim\BG(r,s,\theta)$. There exists a choice of $c,C>0$ such that if
\begin{equation}\label{reqs}
C\frac{\log(r/\delta)}{r}
\le\theta\le
\frac{c}{\sqrt{s}+\log(r/\delta)}
\end{equation}
then $\prob\big(\max_{j}\frac{m_j}{|I_j|}\geq\frac19\big)\leq\delta$.
\end{lemma}
\begin{proof}
${|I_j|\sim\Bin(r,\theta)}$, so corollary \ref{cor:weakbound} of section \ref{concentration} implies that ${\prob(\min_j|I_j|< \frac12r\theta)\leq\delta/s}$ when $\log(s/\delta)\leq r\theta/8$. The latter is satisfied because \eqref{reqs} implies $s\leq r^2$, hence $r\theta\geq C\log(r/\delta)\geq \frac{C}2\log(s/\delta)$.
Let $\cI_j$ be the event $|I_j|\geq \frac12r\theta$ and let $\cR_i$ be the event that ${|R_i|\leq w:= \frac54s\theta+6\log(r/\delta)}$. Then $\prob(\cI_j^c)\leq\delta/s$, and $\prob(\cR_i^c)\leq \delta/r$ by corollary \ref{cor:weakbound}. Let $\cE_{j,i}$ be the event that
\begin{equation}{|C_{j,i}|}\leq \frac54w\theta|I_j|+{\frac{6|I_j|}{r\theta/2}\log(rs/\delta)}\label{thisthing}\end{equation}
Then by lemma \ref{submatrixlemma},
$\prob(\cE_{j,i}^c|\cR_i,\cI_j)\leq \frac{\delta}{rs}$.
Now we bound
\begin{align*}
\prob\Big(\bigcup_{i,j}\cE_{j,i}^c\Big)
=&
\prob\Big(\bigcup_{i,j}(\cE_{j,i}^c\cap\cR_i\cap\cI_j)\Big)
+\prob\Big(\bigcup_{i,j}(\cE_{j,i}^c\cap(\cR_i\cap\cI_j)^c)\Big)
\\\leq&
\sum_{i,j}\prob(\cE_{j,i}^c|\cR_i\cap\cI_j)
+\prob\Big(\big(\bigcup_{i}\cR_i^c\big)\cup\big(\bigcup_{j}\cI_j^c\big)\Big)
\\\leq&
3\delta
\end{align*}
On event $\cE_{j,i}$ we have
\begin{align}\label{fourterms}\frac{|C_{j,i}|}{|I_j|}&\leq \frac54w\theta+\frac{12\log(rs/\delta)}{r\theta}
\\&=\frac{25}{16}s\theta^2+\frac{15}2\theta\log\frac r\delta+\frac{12\log(rs/\delta)}{r\theta}\le\frac19\end{align}
hence on event $\bigcap_{i,j}\cE_{j,i}$, $\bA$ satisfies $\max_i|C_{j,i}|=m_j\leq \frac19|I_j|$ for all $j=1,\ldots,s$. Here the bounds on $s\theta^2$ and $\log(r/\delta)\theta$ follow from the upper bound on $\theta$ in \eqref{reqs}. The bound on $\frac{\log(rs/\delta)}{r\theta}$ follows from $r\theta\geq C\log(r/\delta)\geq\frac{C}2\log(s/\delta)$, which we have used above.\end{proof}
As a corollary we get:
\newtheorem{OCthm}{Theorem \ref{thm:OC_theorem}}
\renewcommand\theOCthm{\unskip}
\begin{OCthm}Let $\bA\sim\BG(r,s,\theta)$. There exists a choice of constants $C,c>0$ such that if
\begin{equation}\label{eq:OC_requirement}
C\frac{\log(r/\delta)}{r}
\le \theta \le 
\frac{c}{\sqrt{s}+\log(r/\delta)}\,,
\end{equation}
then $\bA$ satisfies condition \OC{\frac{r\theta}{10}} with probability at least $1-\delta$, 
  \end{OCthm}
\begin{proof}
  By theorem \ref{thm:Okamoto}, $|I_j|\sim\Bin(r,\theta)$ satisfies that
  $\prob(|I_j|<\frac9{10}r\theta)\leq\exp(-cr\theta)$.
  Furthermore, $|I_j\setminus\essenS_j|\sim\Bin(s,\tilde\varepsilon\theta)$ where
  \[\tilde\varepsilon=\frac1{\sqrt{2\pi}}\int_{-1/10}^{1/10}e^{-t^2/2}dt\leq\frac1{5\sqrt{2\pi}}\approx0.08.\]
  By theorem \ref{thm:Okamoto} again, this and the fact that $\tilde\varepsilon<1/9$ imply that $\prob(|I_j\setminus \overline I_j|> r\theta/9)\leq\exp(-cr\theta)$. By a union bound it holds that with probability $1-O(s\exp(-cr\theta))\geq1-\delta/2$, $|I_j|<\frac9{10}r\theta$ and $|I_j\setminus\overline I_j|<r\theta/9$ for all $j$ simultaneusly. Combine with the bound $m_j\leq|I_j|/9$ for all $j$ from lemma \ref{lem:overlapbound} and we have
\[6m_j+2|I_j\setminus\overline I_j|+\frac{r\theta}{10}\leq\Big(\frac23+\frac29+\frac19\Big)|I_j|,\]
i.e., \OC{r\theta/10} holds.
  \end{proof}

\paragraph{Bound on $h_\varepsilon(\bA)$.}
\label{DC_subsection}
The quantity $h_\varepsilon(\bA)$ is a uniform bound on $L_{\varphi,\varepsilon}(\gamma_{i_1},\gamma_{i_2})$ for all $i_1,i_2$ such that $|R_{i_1}\cap R_{i_2}|=1$. To prove such a bound we condition on $\rho_{i_1}$ and $\rho_{i_2}$ and consider $k\notin I_j$. Then $\gamma_{i_1}(k)=\rho_{i_1}\cdot\rho_k$ depends on $\rho_k(R_{i_1}\setminus\{j\})$, and $\gamma_{i_2}(k)$ depends on $\rho_k(R_{i_2}\setminus\{j\})$. Considering all $k\in\{1,\ldots,r\}\setminus\{i_1,i_2\}$ together, $\gamma_{i_1}$ depends on the columns of $\bA$ indexed by $R_{i_1}\setminus\{j\}$ and $\gamma_{i_2}$ on the columns indexed by $R_{i_2}\setminus\{j\}$. Since $R_1$ and $R_2$ are disjoint, we can view the two submatrices $\bA(\cdot\times (R_{i_1}\setminus j))$ and $\bA(\cdot\times (R_{i_2}\setminus j))$ as independent random matrices $\bA'$ and $\bA''$. In the following lemma, think of $w$ as $|R_{i_1}|\vee|R_{i_2}|$ and of $\rho^{(i)}$ as $\rho_{i}(R_{i}\setminus j)$ for $i=i_1,i_2$.
\begin{lemma}\label{DClemma}
Fix $\rho^{(i)}\in\R^w$ for $i=1,2$. Let $\bA',\bA''\sim\BG(r,w,\theta)$ be independent, and write $\gamma'=\bA'\rho^{(1)}$ and $\gamma''=\bA''\rho^{(2)}$.
With probability $1-\delta$, all $\varphi\in\R\setminus\{0\}$ satisfy that
\[|L_{\varphi,\varepsilon}(\gamma',\gamma'')|\leq rw^2\theta^2\varepsilon+6\log\frac{r}{\delta}\]
\end{lemma}
\begin{proof}
  Write $\bA'=(\rho'_1,\ldots,\rho'_r)^\top$ and $\bA'=(\rho''_1,\ldots,\rho''_r)^\top$. For $k=1,\ldots,r$, write $R'_k=\supp\rho'_k$ and $R''_k=\supp\rho''_k$. We proceed to estimate
\[{p_\varepsilon=\prob(k\in L_{\varphi,\varepsilon}(\gamma',\gamma''))=\prob\Big(e^{-\varepsilon}\leq\varphi^{-1}\frac{\gamma''(k)}{\gamma'(k)}\leq e^{\varepsilon}\Big)}\]
Conditioned on $R'_k$ and $R''_k$, it holds that ${(\gamma'(k),\gamma''(k))\sim\cN(0,\sigma_{R'_k}^2\oplus\sigma_{R''_k}^2)}$, where $\sigma_{R'}^2=\sum_{j\in R'}{\rho_j}^2$ and
\[\sigma_{R_k'}^2\oplus\sigma_{R_k''}^2=\begin{pmatrix}\sigma_{ R_k'}^2&0\\0&\sigma_{ R_k''}^2\end{pmatrix}\]
For $\sigma_{R_k'},\sigma_{ R_k''}\neq0$ we write
\[(x,y)=\Big(\frac{\gamma'(k)}{\sigma_{ R_k'}},\frac{\gamma''(k)}{\sigma_{ R_k''}}\Big)\]
and $\tilde\varphi=\varphi{\sigma_{ R_k'}}/{\sigma_{ R''_k}}$. Then $(x,y)\sim\cN(0,\bI_2)$ conditioned on $ R_k', R_k''$, and it holds that
\begin{align*}
&\prob\Big(e^{-\varepsilon}\leq\varphi^{-1}\frac{\gamma'(k)}{\gamma(k)}\leq e^{\varepsilon}\Big|R_k', R_k''\Big)
\\=&
\prob\Big(e^{-\varepsilon}\leq\tilde\varphi^{-1}\frac yx\leq e^{\varepsilon}\Big|R_k', R_k''\Big)
  \\=&
\frac1\pi\big(\tan^{-1}(e^\varepsilon\varphi')-\tan^{-1}(e^{-\varepsilon}\varphi')\big)
\\=&
  \frac1\pi\big[\tan^{-1}\circ\exp\big]_{(\log\varphi')-\varepsilon}^{(\log\varphi')+\varepsilon}
  \\\leq&\frac\varepsilon\pi
  \label{angle}\end{align*}
Here we have used that the distribution of $(x,y)$ is rotationally invariant, and that
\[\frac{d}{dt}\tan^{-1}\circ\exp(t)=\frac{1}{e^{-1}+e^t}\leq\frac12\]
The law of total probability yields
\begin{equation}\label{p_bound}p_\varepsilon=\prob\Big(e^{-\varepsilon}\varphi\leq\frac{\gamma''(k)}{\gamma'(k)}\leq e^{\varepsilon}\varphi\Big)\leq\frac\varepsilon\pi\prob( R'_k, R''_k\neq\emptyset)\leq \frac\varepsilon\pi(w\theta)^2\end{equation}
For any $\varphi\neq0$, pick $k\in L_{\varphi,\varepsilon}(\gamma',\gamma'')$ and write $\varphi_k=\gamma''_k/\gamma'_k$. Then ${L_{\varphi,\varepsilon}(\gamma',\gamma'')\subset L_{\varphi_k,2\varepsilon}(\gamma',\gamma'')}$. In particular $|L_{\varphi,\varepsilon}(\gamma',\gamma'')|\leq\max_{k=1\ldots r}|L_{\varphi_k,2\varepsilon}(\gamma',\gamma'')|$, which implies that for any $t$,
\begin{equation}\label{maxbound}\prob\Big(\exists\varphi:|L_{\varphi,\varepsilon}(\gamma',\gamma'')|\geq t\Big)
  \leq\prob\Big(\max_{k=1\ldots r}|L_{\varphi_k,2\varepsilon}(\gamma',\gamma'')|\geq t\Big)\end{equation} 
Conditioning on the value of $\varphi_k$, we have that ${|L_{\varphi_k,2\varepsilon}(\gamma',\gamma'')|-1\sim Bin(r-1,p_{2\varepsilon})}$. Then by corollary \ref{cor:weakbound} we have,
\[\prob\big(|L_{\varphi_k,2\varepsilon}(\gamma',\gamma'')|\geq t\:\big|\:\varphi_k\big)\leq \frac\delta{r}\hspace{1cm}t=\frac54rp_{2\varepsilon}+6\log\frac r\delta\]
Apply a union bound over $k=1,\ldots,r$ and insert into \eqref{maxbound} to get that ${\prob (\exists\varphi:|L_{\varphi,\varepsilon}(\gamma',\gamma'')|\geq t)\leq\delta}$. The result follows by inserting \eqref{p_bound} to get ${t=\frac5{2\pi} r\varepsilon w^2\theta^2+6\log\frac{2}\delta}$.
\end{proof}
We let $w$ be on the scale $r\theta$ in lemma \ref{DClemma} to get:
\newtheorem{DCthm}{Theorem \ref{thm:DC_theorem}}
\renewcommand\theDCthm{\unskip}
\begin{DCthm}
Let $\bA\sim\BG(r,s,\theta)$. There exists a choice of constants $C,c>0$ such that if \eqref{eq:OC_requirement} holds, then $h_\varepsilon(\bA)\leq\frac{r\theta}{10}$ with probability $1-\delta$, where
\begin{equation}\label{epsilonchoice}{\varepsilon=\frac{1}{s^2\theta^3+\log(r/\delta)}}.\end{equation}
\end{DCthm}
\begin{proof}  
  Let $E_{i_1,i_2}(h)$ be the event that $|R_{i_1}\cap R_{i_2}|=1$ and
  \[{\sup_{\varphi\in\R}|L_{\varphi,\varepsilon}(\gamma_{i_1},\gamma_{i_2})\setminus I_j|\geq h}\]
  where $\{j\}=R_{i_1}\cap R_{i_2}$. Then we need to show that $\prob(\bigcup_{i_1\neq i_2} E_{i_1,i_2}(r\theta/10))\leq \delta$.
  Let $w=s\theta+\sqrt{6s\theta\log(r/\delta)}$, and define the events
  \[\Omega_{i_1,i_2}=\{|R_{i_1}|\vee|R_{i_2}|\leq w\}\spc\Omega=\{|R_i|\leq w\:\forall i=1,\ldots,r\}\]
$|R_i|\sim\Bin(s,\theta)$, so $\prob(\Omega)\geq1-\delta$ by \eqref{eq:Okamotologeq} following corollary \ref{thm:Okamoto}.
Fix a pair $i_1\neq i_2$ and condition on $\rho_{i_1},\rho_{i_2}$ such that ${|R_{i_1}\cap R_{i_2}|=\{j\}}$ and ${|R_{i_1}|,|R_{i_2}|\leq w}$. Pick disjoint subsets $R',R''$ with $|R'|,|R''|\leq w$ such that $R_{i_1}\setminus\{j\}\subset R'$ and $R_{i_2}\setminus\{j\}\subset R''$.
Define the restricted vectors $\rho=(\rho_{i_1})_{R'}$ and $\rho'=(\rho_{i_2})_{R''}$, write
  \[\bA'=\bA(\{i_1,i_2\}^c\times R')\hspace{1cm}\bA''=\bA(\{i_1,i_2\}^c\times R''),\]
  where $\{i_1,i_2\}^c$ denotes the complement, and apply lemma \ref{DClemma} to obtain that with probability $1-\delta/r^2$,
  \begin{align}\label{L_bound}
    \sup_{\varphi\in\R'}|L_{\varphi,\varepsilon}(\gamma',\gamma'')|&\leq rw^2\theta^2\varepsilon+3\log\frac{r^3}\delta
    \\&\lesssim \varepsilon r s^2\theta^4+(\varepsilon rs\theta^3+1)\log\frac r\delta
    \\&\leq \varepsilon r\theta\big(s^2\theta^3+\log\frac r\delta\big)+\log\frac r\delta
  \end{align}
  here we have used $w^2\lesssim s^2\theta^2+s\theta\log(r/\delta)$ and $s\theta^2\leq1$. The right-hand side is bounded by $h=\frac{r\theta}{10}$ by the conditions $\varepsilon\le\frac{c}{s^2\theta^2+\log(r/\delta)}$ and $\frac{C\log(r/\delta)}r\le\theta$. By the law of total probability applied to $\rho_{i_o}$ and $\rho_{i_2}$, it holds that
$\prob\big(E_{i_1,i_2}(\tfrac{r\theta}{10})\big|\Omega_{i_1,i_2}\big)\leq \frac\delta{r^2}$. We can now bound the probability that $h_\varepsilon(\bA)>r\theta/10$ in the following way:
  \begin{align*}\prob\big(\bigcup_{i_1\neq i_2}E_{i_1,i_2}(\tfrac{r\theta}{10})\big)
    &\leq\prob(\Omega^c)+\sum_{i_1\neq i_2}\prob(E_{i_1,i_2}(\tfrac{r\theta}{10})\cap\Omega)
  \\&\leq\prob(\Omega^c)+\sum_{i_1\neq i_2}\prob(E_{i_1,i_2}(\tfrac{r\theta}{10})\cap\Omega_{i_1,i_2})
  \\&\leq\prob(\Omega^c)+\sum_{i_1\neq i_2}\prob(E_{i_1,i_2}(\tfrac{r\theta}{10})|\Omega_{i_1,i_2})\\&\leq\delta+\sum_{i_1,i_2=1}^r\frac\delta{r^2}=2\delta
  \end{align*}
\end{proof}


\section{Robustness to perturbation}\label{sec:robust}
\newtheorem{noiselemma}{Lemma \ref{lem:noiselemma}}
\renewcommand\thenoiselemma{\unskip}
\begin{noiselemma}
  Let $0<\varepsilon<1$. Let $\gamma_1,\gamma_2\in(\R\setminus[-c,c])^K$ , and let $\overline\gamma_i\in\R^K$, $i=1,2$, be such that
  \[{|\overline\gamma_i-\gamma_i|_\infty\leq (1-e^{-\varepsilon/2}})c\text{,}\]
Then,
\[L_{\varphi,0}(\gamma_1,\gamma_2)\subseteq L_{\varphi,\varepsilon}(\overline\gamma_1,\overline\gamma_2)\text{.}\]
\end{noiselemma}
\begin{proof}
  Let $\varphi=y/x$. We need to show that if $|x|\wedge|y|\geq c$ and $|\overline x-x|\wedge|\overline y-y|\leq(1-e^{-\varepsilon/2})c$, then
  \[e^{-\varepsilon}\leq\varphi^{-1}\frac{\ol y}{\ol x}=\frac{x}{y}\frac{\ol y}{\ol x}\leq e^{\varepsilon}.\]
  Taking the logarithm of the middle expression gives us
  \begin{align*}\Big|\frac{y}{x}\frac{\overline x}{\overline y}\Big|&\leq\Big|\log\frac{\overline x}x\Big|+\Big|\log\frac{\overline y}y\Big|\\&=\Big|\log\big(1+\frac{\overline x-x}x\big)\Big|+\Big|\log\big(1+\frac{\overline y-y}y\big)\Big|\leq \varepsilon,
  \end{align*}
  which finishes the proof of the lemma.
\end{proof}

\newtheorem{noisetheorem}{Theorem \ref{thm:one_column}}
\renewcommand\thenoisetheorem{\unskip}
\begin{noisetheorem}
  Let $\bA\in\R^{r\times s}$ satisfiy \OC{h_{2\varepsilon}(\bA)}. Let $\overline\bSigma$ be such that $|\overline\bSigma-\bA\bA^\top|_\infty\leq (1-e^{-\varepsilon/4})c$, where $c$ is as in definition \ref{def:L_def}. Let $(i_1,i_2)\in\bSigma_j$, and suppose $|\ba_j(i_1)|\wedge|\ba_j(i_2)|\geq 10c$. Then ${\hat\ba=\subalgname(\overline\bSigma,(i_1,i_2),c,\varepsilon)}$ satisfies
  \[\min_{\sigma=\pm1}|\sigma\hat \ba-\ba_j|_\infty=O(|\bA|_\infty|\ba_j|_\infty\varepsilon).\]
\end{noisetheorem}
\begin{proof}
  Assume for notational convenience that $i_1=1$ and $i_2=2$. We also assume for simplicity that $|\gamma_{1}(k)|,|\gamma_2(k)|$ are not in the interval $[c-\varepsilon,c+\varepsilon]$. Let $\varphi=\ba_j(2)/\ba_j(1)$. As in the proof of theorem \ref{thm:correctness} we have that $I_j\setminus(C_{j,1}\cup C_{j,2})\subset L_\varphi(\gamma_{1},\gamma_{2})$. Define the subset
  \[\essenS'=I_j\setminus(C_{j,1}\cup C_{j,2}).\]
  Assumption $|\ba_j(1)|\wedge|\ba_j(2)|\geq 10c$ implies that $|\gamma_i(k)|=|\ba_j(i)\ba_j(k)|\geq c$ for $i=1,2$, and $k\in \essenS_j\setminus(C_{j,1}\cup C_{j,2})$ (recall that $\essenS_j=\{k\big|\:|\ba_j|\geq\frac1{10}\}$). Hence it holds that
  \begin{equation}\label{eq:S_inclusion}\essenS'\subset L_{\varphi,0}(\gamma_{1},\gamma_{2})\cap K\subset L_{\varphi,\varepsilon/2}(\overline\gamma_{1},\overline\gamma_{2}),\end{equation}
  where $K=\{k\in[r]:|\gamma_1(k)|\wedge|\gamma_2(k)|\geq c\}$ and the last inclusion is lemma \ref{lem:noiselemma}.
Pick $\varphi'=\pm e^{\varepsilon m}$ with $m\in\Z$ such that
  \[e^{-\varepsilon/2}\leq\varphi/\varphi'\leq e^{\varepsilon/2}.\]
  Then $L_{\varphi,\varepsilon/2}(\gamma_{1},\gamma_{2})\subset L_{\varphi',\varepsilon}(\gamma_{1},\gamma_{2})$, which combined with \eqref{eq:S_inclusion} yields $\essenS'\subset L_{\varphi',\varepsilon}(\overline\gamma_{1},\overline\gamma_{2})$.
  Condition \OC{h_\varepsilon(\bA)} implies that $|\essenS'|\geq|\essenS_j|-2m_j>h_\varepsilon(\bA)$, so $|L_{\varphi',\varepsilon}(\overline\gamma_{1},\overline\gamma_{2})\cap K|>h_\varepsilon(\bA)$. It must then also hold that
  \begin{equation}|L_{\hat\varphi,\varepsilon}(\overline\gamma_1,\overline\gamma_2)\cap K|>h_\varepsilon(\bA),\label{eq:L_h}\end{equation}
  since $\hat\varphi$ maximizes the LHS among $\varphi'\in\pm e^{\varepsilon\Z}$. We now show that
  \begin{equation}\label{eq:32epsilon}e^{-\frac32\varepsilon}\leq\hat\varphi/\varphi\leq e^{\frac32\varepsilon}.\end{equation}
  We let $\varphi'$ be such that $|\log(\varphi'/\varphi)|>\frac32\varepsilon$, and show that $\varphi'\neq\hat\varphi$. It holds that $L_{\varphi',\varepsilon}(\overline\gamma_1,\overline\gamma_2)$ is disjoint from $L_{\varphi,\varepsilon/2}(\overline\gamma_1,\overline\gamma_2)$ and hence also from $\essenS_j$, a subset of the latter by \eqref{eq:S_inclusion}. By definition \ref{def:h_definition} it holds that $L_{\varphi',\varepsilon}(\overline\gamma_1,\overline\gamma_2)\leq h_\varepsilon(\bA)$, and comparing with \eqref{eq:L_h} shows that $\varphi'\neq\hat\varphi$, proving \eqref{eq:32epsilon}.

  From \eqref{eq:32epsilon} it follows that $L:=L_{\hat\varphi,2\varepsilon}(\overline\gamma_1,\overline\gamma_2)$ contains $L_{\varphi,\varepsilon/2}(\overline\gamma_1,\overline\gamma_2)$, which combinined with \eqref{eq:S_inclusion} yields
  \begin{equation}\essenS'\subset L.\label{eq:between}\end{equation}
 Recalling definition \ref{def:h_definition}, we also get that $|L|\leq|I_j|+h_{2\varepsilon}(\bA)$.

 Let $i$ be as in line \ref{loop} of $\subalgname$ and let
 \[\essenS^i=\essenS_j\setminus(C_{j,1}\cup C_{j,2}\cup C_{j,i})\subset L.\]
 The set inclusion follows from \eqref{eq:between}. Then all entries of $\gamma_i(\essenS^i)/\gamma_1(\essenS^i)$ are $\psi=\ba_j(i)/\ba_j(1)$. But $|\essenS^i|$ contains more than half the elements of $L$,
 \begin{align*}|\essenS^i|&\geq |\essenS_j|-3m_j
 >\frac{|I_j|+h_{2\varepsilon}(\bA)}2\geq\frac{|L|}2.\end{align*}
Here, the strict inequality is exactly condition \OC{h_{2\varepsilon(\bA)}}. It follows that
\[\tilde\ba(i)=\operatorname{median}(\overline\gamma_i(L)/\overline\gamma_1(L))\in\Big[\min_{k\in\essenS^i} \overline \psi(k),\max_{k\in\essenS^i} \overline \psi(k)\Big],\]
where $\overline \psi(k)=\overline\gamma_i(k)/\overline\gamma_1(k)$. Let $\psi=\ba_j(i)/\ba_j(1)$. Then $|\overline\gamma_i-\gamma_i|_\infty\leq c\varepsilon$ and $|\gamma_1(k)|,|\gamma_i(k)|=\Omega(1)$, imply that for $k\in \essenS^i$
\begin{align*}|\overline \psi(k)-\psi|
  &\leq\Big|\frac{(\gamma_1(k)-\overline\gamma_1(k))\overline\gamma_i(k)}{\gamma_1(k)\overline\gamma_1(k)}\Big|+\Big|\frac{\overline\gamma_i(k)-\gamma_i(k)}{\gamma_1(k)}\Big|
  \\&=O\Big(\frac{|\gamma_i(k)|}{|\gamma_1(k)|}+\varepsilon\Big)
  \\&=O\Big(\frac{|\ba_j(i)|}{|\ba_j(1)|}\varepsilon+\varepsilon\Big),\end{align*}
Apply the fact that $\tilde\ba(i)$ is in an interval containing $\overline\psi(k),k\in\essenS^i$ to get $|\ba_j(1)\tilde\ba(i)-\ba_j(1)\psi|$, which translates into
\begin{equation}\label{eq:secondtolast}|\ba_j(1)\tilde\ba-\ba_j|_\infty=O(\varepsilon|\ba_j|_\infty),\end{equation}
The algorithm outputs $\hat\ba=\lambda\tilde\ba$, where
\begin{align*}
  \lambda&=\sqrt{\overline\Sigma_{12}/\tilde\ba(2)}
  \\&=\sqrt{\frac{\ba_j(2)\ba_j(1)+O(\varepsilon)}{\ba_j(2)/\ba_j(1)+O(\varepsilon/\ba_j(1))}}
  \\&=\Big(1+O\big(\frac{\varepsilon}{\ba_j(1)\ba_j(2)}+\frac{\varepsilon}{\ba_j(2)}\big)\Big)\ba_j(1)
  \\&=|\ba_j(1)|+O\big((1+|\ba_j(1)|)\varepsilon\big),\end{align*}
where we have used that $\ba_j(1),\ba_j(2)=\Omega(1)$. Combining with \eqref{eq:secondtolast} yields that
\begin{align*}|\lambda\tilde\ba-\sigma\ba_j|_\infty&=O((1+|\ba_j(1)|)\varepsilon|\ba_j|_\infty+\varepsilon|\ba_j|_\infty).\\&=O(\varepsilon|\bA|_\infty|\ba_j|_\infty),\end{align*}
where $\sgn\ba_j(1)$. Since the output of $\subalgname$ is $\lambda\tilde\ba$, this finishes the proof.
\end{proof}


\section{Concentration inequalities}
\label{concentration}
We use the following bounds from \citet{Dud78}:
\begin{theorem}[Okamoto]\label{thm:Okamoto}
Let $X\sim \Bin(n,p)$ with $p\leq 1/2$ and let $0\leq \varepsilon\leq 2p$. Then
\begin{align}
  P\Big(\frac{X}{n}\leq& p-\varepsilon\Big)\leq \exp\Big(-\frac{n\varepsilon^2}{2p}\Big)
  \\P\Big(\frac{X}{n}\geq& p+\varepsilon\Big)\leq \exp\Big(-\frac{n\varepsilon^2}{6p}\Big)
\end{align}
\end{theorem}
Setting $\varepsilon=\sqrt{\frac{6p}{n}\log\frac1\delta}$ yields the restatement of theorem \ref{thm:Okamoto} that for ${X\sim \Bin(n,p)}$,
\begin{equation}\prob\Big(X\leq np-\sqrt{2np\log(1/\delta)}\Big)\vee\prob\Big(X\geq np+ \sqrt{6np\log(1/\delta)}\Big)\leq\delta\label{eq:Okamotologeq}\end{equation}
We also get a weaker bound which in some settings is more convenient
\begin{corollary}\label{cor:weakbound}
For $X\sim\Bin(n,p)$,
\[\prob\Big(X\leq\frac34np+2\log\frac1\delta\Big)\vee\prob\Big(X\geq\frac54np+6\log\frac1\delta\Big)\leq \delta\]
\end{corollary}
\begin{proof}
  Use the inequality $xy\leq x^2/4+y^2$, which implies $x^2-xy\geq \frac34x^2-y^2$ and $x^2+xy\leq \frac54x^2+y^2$. Let $x=\sqrt{np}$ and set $y$ equal to $\sqrt{2\log(1/\delta)}$ and $\sqrt{6\log(1/\delta)}$ respectively to get
  \begin{align}
    np-\sqrt{2np\log\tfrac1\delta}&\geq\frac34np-2\log\frac1\delta
    \\
    np+\sqrt{6np\log\tfrac1\delta}&\leq\frac54np+6\log\frac1\delta
  \end{align}
  The result now follows from \eqref{eq:Okamotologeq}.
\end{proof}

\begin{corollary}\label{cor:AAbound}
Let $\bA\sim\BG(r,s,\theta)$. Then
\[\prob(|\bA\bA^\top|_\infty> 5s\theta)\leq re^{-s\theta/6}\]
\end{corollary}
\begin{proof}
 Conditioned on $|R_i|=w$, $\|\rho_i\|_2^2\sim\chi_w^2$ is chi-squared with $w$ degrees of freedom. Lemma 1 of \citet{LauMas00} then implies
\[\prob\big(\|\rho_i\|_2^2>w+2\sqrt{wt}+2t\big|\:|R_i|=w\big)\leq e^{-t}\]
We use the inequality $2\sqrt{wt}\leq w+t$ to replace the bound above by $2w+3t$, i.e. $\prob\big(\|\rho_i\|_2^2>2w+3t\big|\:|R_i|=w\big)\leq e^{-t}$. In particular,
\begin{equation}\label{Bernstein}\prob\big(\|\rho_i\|_2^2>4s\theta+3t\big|\:|R_i|\leq 2s\theta\big)\leq e^{-t}\end{equation}
Since $|R_i|\sim\Bin(s,\theta)$, theorem \ref{thm:Okamoto} implies that
$\label{supportconcentration}\prob(|R_i|\geq2s\theta)\leq e^{-{s\theta}/{6}}$. Combining this with \eqref{Bernstein} yields that $\prob(\|\rho_i\|_2^2>5s\theta)\leq e^{-s\theta/6}$. Since $|\bA\bA^\top|_\infty=\max_i\|\rho_i\|_2^2$, a union bound over $i$ yields the result.
\end{proof}

\paragraph{Concentration of sample covariances.}
The sample covariance $\overline\bSigma_n$ follows a Wishart distribution $\mathcal W(\bSigma,n)$. In this section we bound the difference $\bSigma_n-\bSigma$ entrywise.
\begin{lemma}\label{concnormal}
  Let $\binom{X_1}{Y_1},\ldots,\binom{X_n}{Y_n}\overset{i.i.d.}{\sim}\cN\left(\binom00,\Sigma\right)$ where $\Sigma=\binom{\sigma_{XX}\:\:\sigma_{XY}}{\sigma_{XY}\:\:\sigma_{YY}}$, and write $\overline\sigma_{XY}=\frac{1}{n}\sum_{i=1}^n X_iY_i$. Then,
\[\prob\bigg(\frac{|\overline\sigma_{XY}-\sigma_{XY}|}{\sigma_{XX}+\sigma_{YY}}\geq \sqrt{t/n}+t/n\bigg)\leq 4\exp(-t)\]
\end{lemma}
\begin{proof}
Use the polarization identity $XY=\frac14((X+Y)^2-(X-Y)^2)$. Define $\sigma_+,\sigma_->0$ by $\sigma_+^2=\frac14\E(X+Y)^2$ and $\sigma_-^2=\frac14\E(X-Y)^2$, and let $Z=\frac{1}{2\sigma_+}(X+Y)$ and $W=\frac{1}{2\sigma_-}(X-Y)$. Then $Z,W\sim\cN(0,1)$ (they might not be independent), and
\[XY=\sigma_+^2Z^2-\sigma_-^2W^2\]
For each $j=1,\ldots,n$ let $Z_j=\frac{1}{2\sigma_+}(X_j+Y_j)$ and $W_j=\frac{1}{2\sigma_-}(X_j-Y_j)$. We can now write
\begin{equation}\label{difference}\overline\sigma_{XY}=\frac1n\sum_{j=1}^nX_jY_j=\frac1n\sum_{j=1}^n\sigma_+^2Z_j^2-\sigma_-^2W_j^2=\frac{\sigma_+^2}nQ_Z-\frac{\sigma_-^2}nQ_W\end{equation}
where $Q_Z=\sum_{j=1}^nZ_j^2$ and $W_n=\sum_{j=1}^nW_j^2$. Since $Q_Z,Q_W\sim\chi_n^2$, lemma 1 of \citet{LauMas00} implies $\prob(|\frac1n Q_Z-1|\geq 2(\sqrt{t/n}+t/n))\leq2\exp(-t)$ and similarly for $Q_W$. Apply \eqref{difference} and $\sigma_+^2-\sigma_-^2=\sigma_{XY}$ to get
\[\prob(|\overline\sigma_{XY}-\sigma_{XY}|\geq 2\sigma^2(\sqrt{t/n}+t/n))\leq 4\exp(-t)\]
where $\sigma^2=\sigma_+^2+\sigma_-^2$. Using the expressions $\sigma_+^2=\frac14(\sigma_{XX}+\sigma_{YY}+2\sigma_{XY})$ and $\sigma_-^2=\frac14(\sigma_{XX}+\sigma_{YY}-2\sigma_{XY})$ we get $\sigma^2=(\sigma_{XX}+\sigma_{YY})/2$.
\end{proof}

\begin{corollary}\label{thm:covariance_concentration}
Let $\overline\bSigma_n\sim \mathcal W(\bSigma,n)$ be a Wishart matrix with scale parameter $\bSigma\in\R^{r\times r}$, and suppose $0<\delta<1/2$, $n\geq2\log(r/\delta)$. Then with probability at least $1-\delta$,
\[\frac{|\overline\bSigma_n-\bSigma|_\infty}{2|\bSigma|_\infty}\leq 3\sqrt{\frac{\log(r/\delta)}n}\]
\end{corollary}
\begin{proof}
Let $\overline\bSigma_{n}=\frac1n\sum_{i=1}^nX_iX_i^\top$ where $X_1,\ldots,X_n\sim\cN(0,\bSigma)$. Apply lemma \ref{concnormal} for each submatrix $\Sigma=\bSigma(\{k,l\}\times\{k,l\})$, substituting $\binom{X_i(k)}{X_i(l)}$ for $\binom{X_i}{Y_i}$, and take a union bound over pairs $k\leq l$ to get
  \[\prob\bigg(\frac{|\overline\bSigma_n-\bSigma|_\infty}{2|\bSigma|_\infty}\geq \sqrt{t/n}+t/n\bigg)\leq 2r(r+1)\exp(-t)\]
Let $t=\log(2r^2/\delta)\leq 2\log(r/\delta)$. Then $\sqrt{t/n}+t/n\leq2\sqrt{t/n}\leq3\sqrt{\frac1n\log(r/\delta)}$.
\end{proof}


\section*{Acknowledgements}
\markboth{Acknowledgements}{Acknowledgements}
P.R. is supported in part by grants NSF DMS-1712596, NSF DMS-TRIPODS- 1740751, DARPA W911NF-16-1-0551, ONR N00014-17-1-2147 and a grant from the MIT NEC Corporation.

\bibliographystyle{plain}
\bibliography{bibliography}
\end{document}